\documentclass{article}

\usepackage[main, preprint]{neurips_2026}

\usepackage[utf8]{inputenc} 
\usepackage[T1]{fontenc}    
\usepackage[colorlinks=true,linkcolor=blue,citecolor=blue,urlcolor=blue]{hyperref}
\usepackage{url}            
\usepackage{booktabs}       
\usepackage{amsfonts}       
\usepackage{nicefrac}       
\usepackage{microtype}      
\usepackage{xcolor}         
\usepackage{amsmath}
\usepackage{algorithm}
\usepackage{algorithmic}

\usepackage{amsthm}
\usepackage{mathtools} 
\usepackage{amssymb}
\usepackage{wrapfig}
\usepackage{enumitem}

\definecolor{COL_FULL}{HTML}{C0392B}
\definecolor{COL_GRAD}{HTML}{E76F51}
\definecolor{COL_SOL}{HTML}{0D695E}

\newtheorem{theorem}{Theorem}[section]

\newtheorem{lemma}[theorem]{Lemma}
\newtheorem{proposition}[theorem]{Proposition}

\theoremstyle{definition}

\theoremstyle{remark}

\author{
  Na\"il B.~Khelifa\\ 
  University of Cambridge\\
  \texttt{nbk24@cam.ac.uk} \\
  \And
  Richard E. Turner\\
  University of Cambridge\\
  \texttt{ret26@cam.ac.uk} \\
  \And
  Ramji Venkataramanan\\
  University of Cambridge\\
  \texttt{rv285@cam.ac.uk} \\
}

\newcommand{\R}{\mathbb{R}}
\newcommand{\E}{\mathbb{E}}
\newcommand{\Prob}{\mathbb{P}}

\newcommand{\KL}{\mathrm{KL}}

\newcommand{\pstar}{p^\star}
\newcommand{\phat}{\hat{p}}

\newcommand{\cI}{\mathcal{I}}

\newcommand{\cD}{\mathcal{D}}
\newcommand{\cF}{\mathcal{F}}
\newcommand{\cG}{\mathcal{G}}

\newcommand{\cL}{\mathcal{L}}
\newcommand{\cJ}{\mathcal{J}}

\newcommand{\cN}{\mathcal{N}}
\newcommand{\cP}{\mathcal{P}}

\newcommand{\bG}{\mathbf{G}}

\newcommand{\bepsilon}{\boldsymbol{\varepsilon}}
\newcommand{\bu}{\mathbf{u}}
\newcommand{\bx}{\mathbf{x}}

\newcommand{\by}{\mathbf{y}}
\newcommand{\bg}{\mathbf{g}}
\newcommand{\bv}{\mathbf{v}}
\newcommand{\bw}{\mathbf{w}}

\newcommand{\cGs}{\mathcal{G}_s}

\newcommand{\bX}{\mathbf{X}}
\newcommand{\bY}{\mathbf{Y}}
\newcommand{\bZ}{\mathbf{Z}}
\newcommand{\br}{\mathbf{r}}
\newcommand{\bB}{\mathbf{B}}

\newcommand{\bF}{\mathbf{F}}

\newcommand{\be}{\mathbf{e}}
\newcommand{\bs}{\mathbf{s}}
\newcommand{\bshat}{\hat{\mathbf{s}}}
\newcommand{\bb}{\mathbf{b}}

\newcommand{\bI}{\mathbf{I}}

\newcommand{\shat}{\hat{\bs}}

\newcommand{\md}{\mathrm{d}}

\newcommand{\law}{\mathrm{Law}}
\newcommand{\Probspace}{\mathcal{P}}

\title{Diffusion Models Observe Only Gradients: A Geometric Perspective on Score Matching Errors}

\begin{document}

\maketitle

\begin{abstract} 
Score-based diffusion models are typically trained by minimizing the $L^2$ score matching error, and standard theoretical analyses rely on this quantity to bound the sampling discrepancy between the learned and target distributions. Recent work has shown that the score admits an orthogonal decomposition into  gradient and solenoidal components in which only the gradient part matters for exact sampling.  We apply this decomposition to the score error of trained models and study its quantitative consequences. Since only the gradient component of the error enters the marginal Fokker–Planck dynamics, a model can incur arbitrarily large $L^2$ score error while perfectly matching the target. Starting from this observation, we present three results. First, we derive an upper bound on the Kullback-Leibler divergence that depends only on the observable gradient component of the error. It tightens the standard Girsanov bound for generic score networks and identifies that bound’s looseness as the cost of working with path-space rather than marginal-space dynamics. Second, we prove that no non-trivial function of the $L^2$ score error can uniformly lower bound any divergence between the learned and target distributions. Third, we give an estimator of the gradient component, computable from denoising score matching residuals  via a dual Sobolev identity, which empirically correlates substantially better with sample quality than the full $L^2$ error.
\end{abstract}

\section{Introduction}
\label{sec:introduction}

Denoising Score Matching (DSM) and its variants \cite{hyvarinen2005estimation, hyvarinen2007_extension, denoising-score-matching-vincent, song2020sliced} have driven the success of score-based diffusion models~\cite{SohlDickstein15_noneq_thermod, DDPM,  song2019generative, song2021scorebased} by providing a scalable score-estimation objective. To approximate an unknown target distribution $\pstar$, a score-based diffusion model transports it to an easy-to-sample prior distribution (typically, a standard Gaussian), from which samples are brought back, via a reverse diffusion, towards $\pstar$. To implement this reverse diffusion, one must estimate the unknown scores $\nabla \log \pstar_s$ of the time-$s$ noised marginals of the target distribution along the reverse path (precise definitions in Section \ref{sec:background}). DSM produces such estimators by minimizing the time-integrated $L^2$ error from the true unknown score. The accuracy of this estimation critically determines the sampling quality of the diffusion model \cite{sampling-is-as-easy-as-learning-the-score}, as illustrated by the following fundamental upper bound (up to constants) on the KL divergence between the target $\pstar$ and the learned distribution $\phat$:
\begin{equation}
\label{eq:intro-upper-bound}
\KL(\pstar \,\| \,\phat)
\;\lesssim\;
\frac{1}{2}\int_{0}^T 
\mathbb{E}\!\left[\|
\be_s\|^2_2\right]\,\md s,
\tag{Full $L^2$}
\end{equation}
where $\be_s = \bshat_{s}-\nabla \log \pstar_s$ is the time-$s$ score estimation error, i.e. the difference between the estimated score $\bshat_{s}$ and the true score $\nabla \log \pstar_s$. Inequality \eqref{eq:intro-upper-bound} underpins much of the current theory of diffusion models \cite{max-likelihood-diffusion-models}, from convergence analyses~\cite{sampling-is-as-easy-as-learning-the-score, convergence-for-score-based-gen-modelling-lee, chen2023improved, benton2024nearly}, to minimax optimality results~\cite{oko2023diffusion, minimax-optim-score-based-diff-models,from-optimal-score-matching-to-optimal-sampling, samworth-shape-constraint}, and interpretations of training objectives in terms of sample quality~\cite{understanding-diffusion-objectives, variational-diff-models}.

Inequality~\eqref{eq:intro-upper-bound} ties the full $L^2$ score error to the sampling quality of the diffusion model. However, as observed in the exact setting by Horvat and Pfister~\cite{gauge-freedom-conservativity} and Vuong et al.~\cite{vuong2025are}, a learned diffusion model can incur arbitrarily large $L^2$ score error while perfectly matching the target distribution:
\begin{equation}
\label{eq:surprise}
\KL(\pstar\,\|\, \phat) = 0 \qquad \text{ while } \qquad \int_{0}^T
\mathbb{E}\!\left[\|\be_s\|^2_2\right]\,\mathrm{d}s \gg 0.
\end{equation}

As a first contribution, we strengthen this into an impossibility result: no monotone function of the $L^2$ score error can uniformly lower-bound any meaningful divergence between $\phat$ and $\pstar$ (Theorem \ref{thm:observable-score-error-principle} (iii)). The $L^2$ error is therefore not the intrinsic quantity controlling marginal sampling quality, raising a natural question:
\begin{center}
\emph{Which components of the score error control sampling quality?}
\end{center}

\begin{figure}[t]
    \centering
    \includegraphics[width=\linewidth]{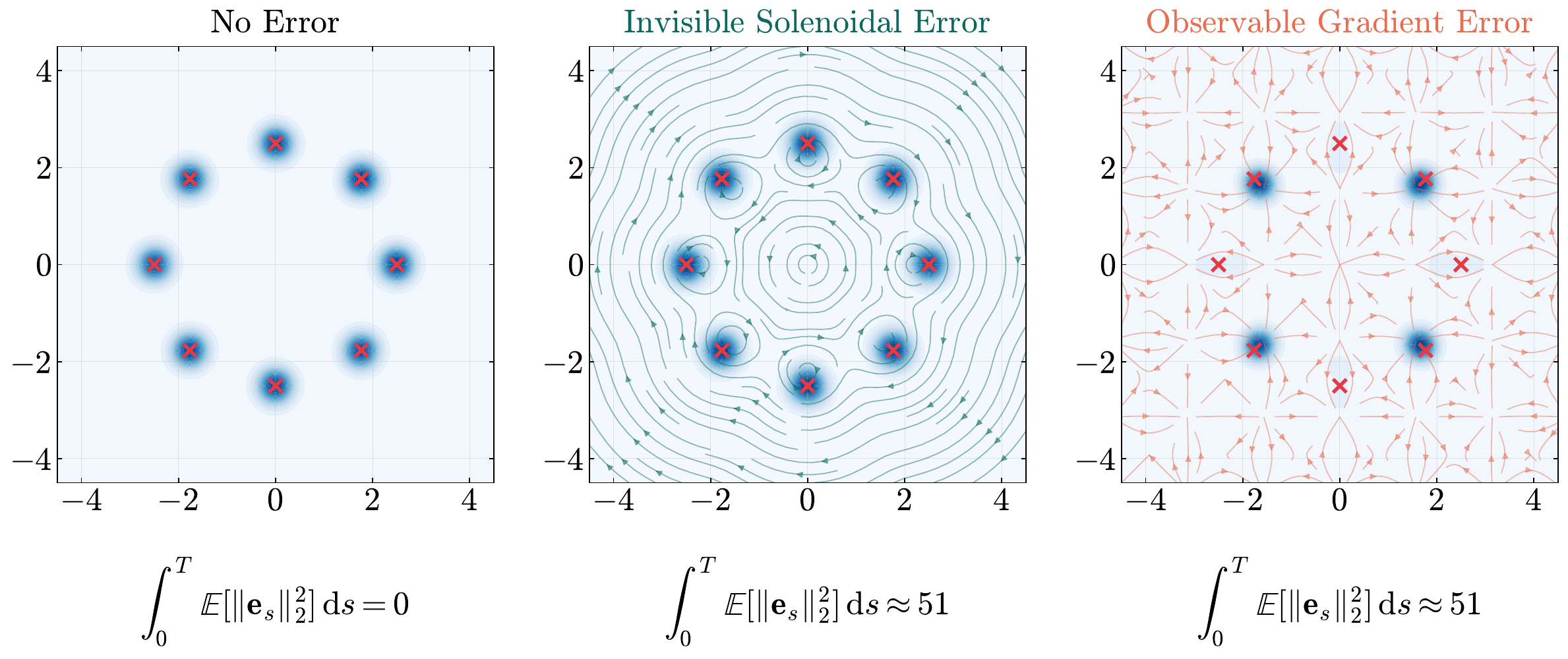}
    \vspace{-0.5cm}
    \caption{\textbf{Helmholtz-Hodge decomposition of score errors.} Two score estimates, both with total $L^2$ score matching error approximately equal to 51, manifest in qualitatively different generated distributions depending on their geometric structure. Red crosses represent true mode locations and blue centroids represent learned modes. \textit{Center}: with purely solenoidal error fields (green arrows), the generated distribution is identical to the no-error case (left). \textit{Right}: with gradient-like error fields (red arrows) of the same energy, the generated distribution visibly misses some modes. Standard score matching minimizes the full $L^2$ error, which cannot distinguish between these two cases, even though only the gradient component affects sampling quality.}
    \label{fig:invisible-errors}
\end{figure}

The key feature underlying observation \eqref{eq:surprise} is that the marginal distributions of the diffusion process used to generate the sample solve a Fokker-Planck equation \cite{variational-fokker-planck} that depends on the score error $\be_s$ only through a weighted divergence (Section \ref{sec:FP-invisible-errors}) . Decomposing the error vector field into \textcolor{COL_GRAD}{gradient} and \textcolor{COL_SOL}{solenoidal} components (a Helmholtz--Hodge decomposition \cite{bhatia2013helmholtz}, standard in scientific machine learning \cite{richterpowell2022neural, hansen2023learning, genuist2026divergence, trupin2026learning, li2026project}), the weighted divergence term filters out the \textcolor{COL_SOL}{solenoidal} component of the error field, making it structurally invisible to the marginal dynamics, while the \textcolor{COL_GRAD}{gradient} component affects marginals (see Figure \ref{fig:invisible-errors}). Therefore, a model with purely solenoidal errors satisfies \eqref{eq:surprise}, as illustrated in the middle panel of Figure \ref{fig:invisible-errors}.  This decomposition and the resulting invisibility were established in the context of diffusion models in~\cite{gauge-freedom-conservativity, vuong2025are} to study exact sampling and density estimation. We instead apply it, for trained models with nonzero error, to the score \emph{error} and use it to obtain quantitative guarantees on sampling quality.

The discussion above suggests that, although DSM reduces the gradient component of the score error by minimizing the full $L^2$ objective, this overall score matching error also penalizes solenoidal components that do not affect marginals. Therefore, it is not always a reliable indicator of sampling quality (as in the first two panels of  Figure \ref{fig:invisible-errors}). Instead, considering only the $L^2$ norm of the observable \textcolor{COL_GRAD}{gradient} component of errors (denoted $\Pi_{\cGs}\be_s$), we derive a new bound on sampling quality (Theorem \ref{thm:kl-upper-new}):
\begin{equation}
\label{eq:kl-upper-new-intro}
\mathrm{KL}(\pstar\,\|\,\phat)
\;\lesssim\;
\frac{1}{2}\int_{t_0}^T \E\big[\|\Pi_{\cGs}\be_s\|^2\big]\md s.
\tag{Ours}
\end{equation}
This new bound is strictly tighter than the standard one in \eqref{eq:intro-upper-bound} whenever the solenoidal component of errors is nonzero, motivating a better diagnostic of learning quality based on the geometry of diffusion. We propose a tractable surrogate to estimate the gradient component of the score
error. 

\paragraph{Contributions.} Our results concern the approximate regime where the score estimation error is nonzero, i.e. the trained score network does not learn the true score  perfectly. 
\begin{itemize}
\item \textbf{An impossibility result.} Complementing the exact gauge-freedom results of~\cite{gauge-freedom-conservativity, vuong2025are}, we show that no strictly increasing function of the full $L^2$ score error can uniformly lower-bound any divergence between the learned and target distributions (Theorem~\ref{thm:observable-score-error-principle}~(iii)).

\item \textbf{A tighter endpoint KL bound.} Applying the Helmholtz--Hodge decomposition to the score error, we prove that the Kullback-Leibler divergence between the learned and target models is controlled only by the $L^2(\pstar_s)$ norm of the gradient component of the error (Theorem~\ref{thm:kl-upper-new}), tightening the standard Girsanov bound~\cite{benton2024nearly, chen2023improved, sampling-is-as-easy-as-learning-the-score}. This component is the minimum-energy error field producing the same marginal evolution. 

\item \textbf{A bridge with path-space analysis.} We recover the same bound by applying Girsanov's theorem to a marginally equivalent representative of the true reverse process (Appendix~\ref{sec:girsanov-theory-link}). This identifies the looseness of the standard Girsanov bound in controlling a marginal-level divergence by a path-space one, which accumulates contributions from all errors, even those that do not affect the marginals.

\item \textbf{A DSM-computable diagnostic.} Using a dual Sobolev identity, we give a tractable estimator for the norm of the gradient component of the score error, i.e.\ the right-hand side of our bound (Section~\ref{subsec:computational-aspect}), that requires only forward-noised data and the frozen score network to be computed. On image datasets, the observable error decreases by two to three orders of magnitude during training, while the residual $L^2$ error is dominated by the invisible component. The observable error also correlates substantially better with feature-FID than the full $L^2$ error on Fashion-MNIST.
\end{itemize}

\paragraph{Related work.} Recent studies have found that standard denoising score matching can be inefficient in various settings. When the data concentrate near a low-dimensional manifold, score matching can waste capacity by forcing the score network to fit large directions normal to the data manifold \cite{pidstrigach2022score, convergence-manifold-hypotheses-vdb}. In another vein, a popular line of work modifies the relative importance of score errors across noise levels or frequency regimes through timestep weighting and training reweighting schemes~\cite{minSNR,choiperception, wang2025an,beta2025sampling}. These studies reflect the empirical fact that not all score errors are equally relevant for sample quality. However, all these approaches treat the full score estimation error as the central quantity to analyze, without examining its geometric structure. 

Another  line of work has debated whether the score network should be constrained to be conservative (a gradient field) by construction. Such a conservative parametrization would indeed automatically lead to errors with no solenoidal components. While early evidence suggested unconstrained parametrizations lose little in practice \cite{should-ebm-model-the-energy-of-the-score}, Chao et al. \cite{on-investigating-conservative-property-sgm} reported that non-conservativity can degrade sampling and proposed a Jacobian-symmetrizing penalty to address this. On the theoretical side, \cite{saremi2019approximatingnablafneural} shows that an unconstrained network generically cannot represent an exact gradient, so a nonzero solenoidal component is
the typical case. Horvat and Pfister~\cite{gauge-freedom-conservativity} and Vuong et al.~\cite{vuong2025are} showed that conservativity of the score network is not required to produce exact marginals, as the solenoidal component is invisible to the continuous-time marginal dynamics. Theorem~\ref{thm:observable-score-error-principle} extends this in the approximate regime, and Theorem~\ref{thm:kl-upper-new} quantifies the effect of the remaining, observable error on the sampling quality.
In concurrent work, Dupuis et al.~\cite[Theorem 4.3]{dupuis2026tighteningscorematchinggap} obtained an upper bound closely related to the one in Theorem \ref{thm:kl-upper-new}, under a  different set of assumptions. We provide an extensive literature review in Appendix \ref{app:related-work}.

\paragraph{Notation.} \label{par:notation}
Bold symbols (e.g.\ $\mathbf f, \bv, \bw, \ldots $) denote $\R^d$-valued objects, and $\R^d$-valued random variables are denoted $\bX$, $\bY, \dots$. $\Probspace(\R^d)$ (resp. $\Probspace_2(\R^d)$) denotes the space of Borel probability measures on $\R^d$ (resp. with finite second moment). For $\mu, \nu \in \Probspace_2(\R^d)$, the Kullback-Leibler divergence between these two measures is defined as $\mathrm{KL}(\mu \,\|\,\nu)=\int_{\R^d} \log(\tfrac{\mu}{\nu})\mu\md \bx$ if $\mu$ is absolutely continuous with respect to $\nu$ and $+\infty$ otherwise. 
For a distribution $\mu \in \Probspace_2(\R^d)$, the space of square integrable vector fields relative to $\mu$ is defined as $L^2(\mu, \R^d):=\{\bv: \R^d \to \R^d~:~\E_{\mu}[\|\bv\|^2_2]<\infty\}$. It is a Hilbert space when equipped with its $L^2$ scalar product defined as follows. For $\bv, \bw  \in L^2(\mu, \R^d)$, we have $\langle \bv\,,\bw \rangle_{L^2(\mu, \R^d)} = \E_\mu[\bv \cdot \bw]$, where $\cdot$ denotes the standard Euclidean scalar product. This geometry yields the $L^2$-norm, defined as $\|\bv\|_{L^2(\mu)} = (\E_{\mu}[\|\bv\|^2_2])^{1/2}$.

 The space of continuous functions from $[0, T]$ to $\R^d$ is denoted by $C([0,T],\R^d)$. For a stochastic process $(\bZ_t)_{t\in[0, T]} \in C([0,T],\R^d)$, $\law((\bZ_t)_t)$ denotes its path-space distribution and $\law(\bZ_t)$ denotes the law of its marginal at time $t$. The space of differentiable scalar-valued functions on $\R^d$ with continuous gradient is denoted by $C^1(\R^d)$, and $C_c^\infty(\R^d)$ denotes the space of smooth (i.e. infinitely differentiable) scalar-valued functions with compact support on $\R^d$.

\section{Background}
\label{sec:background}

\paragraph{Forward diffusion.}

A forward diffusion on $[0, T]$ ($T > 0$) is defined by the stochastic differential equation (SDE):
\begin{equation}
\label{eq:forward-sde}
\md\bX_t = \boldsymbol{f}_t(\bX_t)\,\md t + \sigma_t\,\md\bB_t,
\qquad \bX_0 \sim \pstar,
\qquad t \colon 0 \to T,
\end{equation}
where $\bB_t$ is a standard $d$-dimensional Brownian motion, $\boldsymbol{f}_t : \R^d \to \R^d$ is the drift, and $\sigma_t > 0$ is the scalar diffusion coefficient. Under Assumption \ref{ass:coeff-reg} below, this equation admits a non-explosive, pathwise unique solution \cite{sto-calculus-2}, and its marginals $\pstar_t := \law(\bX_t)$ satisfy the forward Fokker–Planck equation \cite{variational-fokker-planck, bogachev_krylov_rockner_shaposhnikov, stroock_varadhan}: 
\begin{equation}
\label{eq:fokker-planck-forward}
\partial_t \pstar_t
= -\nabla\!\cdot\!\big(\boldsymbol{f}_t\,\pstar_t\big)
+ \frac{\sigma_t^2}{2}\,\Delta \pstar_t,
\qquad \pstar_0 = \pstar.
\end{equation}
Because the data distribution $p^\star$ may be non-smooth or even singular, the score $\nabla \log p_t^\star$ can become unstable as $t\downarrow 0$ \cite{score-approx-chen-low-dim,zhang2024tackling}. Therefore, as is common in the literature \cite{khelifa2026errorpropagationmodelcollapse, benton2024nearly, sampling-is-as-easy-as-learning-the-score}, we work on a truncated interval $[t_0, T]$ with $t_0 >0$, where the diffusion has already regularized the law: under suitable nondegeneracy and regularity assumptions  \cite{aronson1967, aronson1968}, $\pstar_t$ admits a smooth density for every $t>0$, so that the score is well-defined and better behaved away from the singular endpoint. 

\paragraph{Reverse-time diffusion.}

Under regularity conditions on the drift and diffusion coefficients (assumed throughout), the forward process admits a time reversal
\cite{time-reversal-of-diffusions, anderson1982reverse, conditions-time-reversal, time-reversal-conditions-1}. The reverse-time SDE, running backwards from $T$ down to $t_0$, is written in the standard backward-time convention $s:T \downarrow t_0$ (so that $\md s < 0$):
\begin{equation}
\label{eq:reverse-sde-true}
\md\bY_s = [\boldsymbol{f}_s(\bY_s) - \sigma_s^2\,\nabla_{\bx} \log \pstar_s(\bY_s)]\,\md s + \sigma_s\,\md\bar{\bB}_s,
\qquad \bY_T \sim \pstar_T,
\qquad s \colon T \downarrow t_0,
\end{equation}
where $\bar{\bB}_s$ is a Brownian motion under the reverse filtration. The marginals of this reverse SDE \eqref{eq:reverse-sde-true} exactly match those of the forward so that $\law(\bY_s)=\pstar_s$ for all $s$, and they satisfy the backward Fokker-Planck equation (once again, backward in time \textit{i.e.} $\md s < 0$): 
\begin{equation}
\label{eq:FP-true-reverse}
\partial_s \pstar_s
= -\nabla\!\cdot\!\Big(\big[\boldsymbol{f}_s - \sigma_s^2\,\nabla_\bx \log \pstar_s\big]\,\pstar_s\Big)
- \frac{\sigma_s^2}{2}\,\Delta \pstar_s, 
\quad \qquad s \colon T \downarrow t_0,
\end{equation}

\paragraph{Score estimation.}

In practice, the true score $\nabla_\bx \log \pstar_s$ is unknown, and is estimated via score matching \cite{hyvarinen2005estimation, hyvarinen2007_extension} or denoising score matching (DSM) \cite{denoising-score-matching-vincent}. In modern diffusion and score-based generative modeling, DSM is used to learn a time-dependent score field $\bs_\theta : \R^d \times [t_0, T] \to \R^d$ (parametrized by $\theta$) of noise-perturbed marginals, by regressing toward the conditional score \cite{song2019generative, song2021scorebased, song2020sliced},
\begin{equation}
\label{eq:dsm_objective}
\cL_{\mathrm{DSM}}(\theta)
= \E_{s \sim \mathcal{U}[t_0, T]}\!\left[\lambda(s)\E_{\bx_0} \E_{\bx_s | \bx_0}
\,\big\|\bs_\theta(\bx_s, s)
- \nabla_{\bx_s}\log \pstar_s(\bx_s \mid \bx_0)\big\|^2_2
\right],
\end{equation}
where $\lambda:[0,T] \to \R_{>0}$ is a continuous positive time weighting. At the population level and under mild regularity, the minimizer of \eqref{eq:dsm_objective} approximates the true score field of the noised marginals. 

\paragraph{Learned reverse process.}

Plugging the learned score into \eqref{eq:reverse-sde-true} defines the learned reverse process
\begin{equation}
\label{eq:reverse-sde-learned}
\md\hat{\bY}_s
= [\boldsymbol{f}_s(\hat{\bY}_s) - \sigma_s^2\,\bs_\theta(\hat{\bY}_s, s)]\,\md s + \sigma_s\,\md\bar{\bB}_s,
\qquad \hat{\bY}_T \sim \pstar_T,
\qquad s \colon T \to t_0.
\end{equation}
Defining  $\be_s(\bx) := \bs_\theta(\bx, s) - \nabla_\bx \log \pstar_s(\bx)$  as the vector field of \textit{score estimation error}, 
the marginals $\hat{p}_s := \law(\hat{\bY}_s)$ of the reverse process \eqref{eq:reverse-sde-learned} satisfy the backward Fokker-Planck equation:
\begin{equation}
\label{eq:FP-learned-reverse}
\partial_s \hat{p}_s
= 
-\nabla\!\cdot\!\Big(\big(\boldsymbol{f}_s - \sigma_s^2\,\nabla_\bx \log \pstar_s\big)\,\hat{p}_s\Big)
-
\frac{\sigma_s^2}{2}\,\Delta \hat{p}_s 
+
\sigma_s^2\,\nabla\!\cdot\!\big(\be_s\,\hat{p}_s\big),
\quad \hat{p}_T = \pstar_T, \quad s \colon T \to t_0.
\end{equation}

The learned generative distribution is $\hat{p}_{t_0} := \law(\hat{\bY}_{t_0})$. Moreover, writing  $\nabla_\bx \log \pstar_s(\bx) = \bs_\theta(\bx, s) -\be_s(\bx)$ in the true reverse Fokker--Planck \eqref{eq:FP-true-reverse} yields, 
\begin{equation}
\label{eq:FP-true-reverse-v2}
\partial_s \pstar_s
= 
-\nabla\!\cdot\!\Big(\big[\boldsymbol{f}_s - \sigma_s^2\,\bs_\theta(\cdot, s)\big]\,\pstar_s\Big)
-
\frac{\sigma_s^2}{2}\,\Delta \pstar_s
- 
\sigma_s^2\nabla\!\cdot\!(\pstar_s\be_s), \qquad s \colon T \downarrow t_0.
\end{equation}

\paragraph{Weighted divergence-free fields.}
For a density $\mu$, a field $\bw \in L^2(\mu;\R^d)$ is called
divergence-free if $\nabla\!\cdot(\mu\bw)=0$ in the weak sense, namely
$\int _{\R^d}\nabla\varphi(\bx)\cdot \bw(\bx)\,\mu(\bx)\,\md\bx=0$ for all
$\varphi\in C_c^\infty(\R^d)$. 

\paragraph{Assumptions.} We make four assumptions, which are standard in the literature on diffusion-model convergence  \cite{score-approx-chen-low-dim, optimal-score-est-via-empirical-bayes-smoothing, minimax-optim-score-based-diff-models, chen2023improved} and required in the proofs of Theorems \ref{thm:observable-score-error-principle} and \ref{thm:kl-upper-new}. 

\begin{enumerate}[label=\textbf{(A\arabic*)}, ref=A\arabic*, leftmargin=*, resume=assumptions]
    \item \label{ass:minimal-assumptions-marginals} \textbf{Marginal regularity.} For all $s \in [t_0, T]$, the marginals $\pstar_s$ and $\phat_s$ are strictly positive, $C^1$ probability densities on $\R^d$ with finite second moment.
    \item \label{ass:coeff-reg} \textbf{Regularity of Drift,  Diffusion coefficient, and Score Errors.} The map $s\mapsto\sigma_s$ is continuous on $[t_0,T]$ with $0<\underline\sigma\le\sigma_s\le\bar\sigma$. The maps $(s,\bx)\mapsto \boldsymbol f_s(\bx)$, $\bs_\theta(\bx,s)$, $\be_s(\bx)$ are (i) jointly Borel measurable, (ii) locally Lipschitz in $\bx$, and (iii) of linear growth: there exists $C<\infty$ such that, for a.e.\ $s\in[t_0,T]$ and all $\bx\in\R^d$,
    $\|\boldsymbol f_s(\bx)\|+\|\bs_\theta(\bx,s)\|+\|\be_s(\bx)\|\le C(1+\|\bx\|)$. 
    \item \label{ass:density-reg} \textbf{Density regularity.} The maps $(s,\bx)\mapsto\pstar_s(\bx)$ and $(s,\bx)\mapsto\phat_s(\bx)$ are of class $C^{1,2}$ on $[t_0,T]\times\R^d$. With $r_s:=\pstar_s/\phat_s$,
    $$
    \int_{t_0}^T\!\!\int_{\R^d}\Big(|\partial_s\pstar_s|\,\big(1+|\log r_s|\big)+ |\partial_s\phat_s| r_s\Big)\md\bx\,\md s<\infty .
    $$
    \item \label{ass:fisher} \textbf{Pointwise score integrability and tails.} For a.e.\ $s\in[t_0,T]$, $\nabla\log\phat_s$ has linear growth, and $\nabla\log\pstar_s,\ \nabla\log\phat_s\in L^2(\pstar_s;\R^d)$.
\end{enumerate}
\paragraph{Role of the assumptions} Assumption~\ref{ass:coeff-reg} gives existence of unique non-explosive solutions to both reverse SDEs, and uniqueness of the associated Fokker--Planck curves. Together with the finite second moment in~\ref{ass:minimal-assumptions-marginals}, it also implies $\be_s\in L^2(\pstar_s)\cap L^2(\phat_s)$. Assumption~\ref{ass:density-reg} justifies differentiating $s\mapsto\KL(\pstar_s\|\phat_s)$ under the integral (Lemma~\ref{lem:derivative-kl-fisher}). Assumption~\ref{ass:fisher} ensures that the relative Fisher information $\cI(\pstar_s\|\phat_s)$ is finite for a.e.\ $s$, and that all integrations by parts are valid (Lemma~\ref{lem:derivative-kl-fisher}).

\section{Geometric Structure of Score Errors: Observable and Invisible Components}
\label{sec:FP-invisible-errors}

The true densities $(\pstar_s)_{s\in[t_0,T]}$ evolve according to the backward Fokker-Planck given by \eqref{eq:FP-true-reverse-v2}, in which the score estimation errors $(\be_s)_{s\in[t_0,T]}$ enter exclusively through the scalar weighted divergence terms $(\nabla \cdot (\pstar_s \be_s))_{s\in[t_0,T]}$. This induces an invariance at the level of marginals: two errors with the same weighted divergences generate the same marginal dynamics. In particular, adding any divergence-free perturbation $\bw_s \in L^2(\pstar_s; \R^d)$, i.e. any $\bw_s$ such that $\nabla \cdot (\pstar_s \bw_s)=0$, leaves the backward Fokker-Planck unchanged. We therefore refer to these locally mass-preserving perturbations as \textit{invisible}: they may contribute to the ambient $L^2$ score error, but they do not affect the reverse marginal dynamics.

For clarity, in this section we fix a time $s \in [t_0, T]$ in the reverse diffusion and study how quantities behave at that time.

\paragraph{Helmholtz-Hodge Decomposition.} To characterize these invisible perturbations, we observe that for any differentiable vector field $\bw_s$, the following are equivalent (by weak integration by parts):
\begin{equation}
   \nabla\!\cdot\!\big(\pstar_s\bw_s\big) = 0
\iff
\text{ for all } \varphi\in C_c^\infty(\R^d), \text{ we have } \langle \nabla \varphi, \bw_s\rangle_{L^2(\pstar_s;\R^d)}=0
\iff \bw_s \in \cGs^\perp,  
\label{eq:divfree_cond}
\end{equation}
where $\cGs := \mathrm{cl}({\{\nabla\varphi:\ \varphi\in C_c^\infty(\R^d)\}})$ and $\mathrm{cl}$ denotes the closure with respect to $L^2(\pstar_s;\R^d)$; the inner product notation in \eqref{eq:divfree_cond} is defined on p.\pageref{par:notation}. Thus, the set of vector fields $\bw_s$ such that $\nabla\!\cdot\!\big(\bw_s\pstar_s\big) = 0$ is the orthogonal complement of gradient fields in the space $L^2(\pstar_s;\R^d)$. We can decompose $L^2(\pstar_s;\R^d)$ into a gradient subspace $\cGs$ and its orthogonal complement $\cGs^\perp$. Expressing the error  using this  decomposition yields the \textit{weighted Helmholtz--Hodge} decomposition \cite{bhatia2013helmholtz}:
\begin{equation}
\label{eq:HH-decomposition-errors}
\be_s=\Pi_{\cGs}\be_s+\Pi_{\cGs^\perp}\be_s \in L^2(\pstar_s;\R^d)=\cGs\oplus \cGs^\perp
\end{equation}
where $\Pi_{\cGs}$ (resp. $\Pi_{\cGs^\perp}$) denotes the orthogonal projection operator on $\cGs$ (resp. on $\cGs^\perp$). The second component is naturally divergence-free and thus structurally invisible at the reference density since $\nabla \!\cdot\! (\pstar_s \be_s) = \nabla\!\cdot\!\big(\pstar_s\Pi_{\cGs}\be_s\big)$. Plugging \eqref{eq:HH-decomposition-errors} in the backward Fokker-Planck \eqref{eq:FP-true-reverse-v2} implies,
\begin{equation}
\label{eq:FP-HH-learned-reverse}
\partial_s \pstar_s
= -\nabla\!\cdot\!\Big(\big[\boldsymbol{f}_s - \sigma_s^2\,\bs_\theta(\cdot, s)\big]\,\pstar_s\Big)
- \frac{\sigma_s^2}{2}\,\Delta \pstar_s
- \sigma_s^2 \nabla\!\cdot\!\big(\pstar_s\Pi_{\cGs}\be_s\big),
\qquad s: T \downarrow t_0.
\end{equation}
Intuitively, the component $\Pi_{\cGs}\be_s$ corresponds to a \textit{transport of mass across space}, which modifies the density. In contrast, $\Pi_{\cGs^\perp}\be_s$ preserves mass locally (e.g. rotation) without affecting the density. As a result, two error fields with identical $L^2$-norm can induce radically different effects on the generated distribution, as illustrated in Figure \ref{fig:invisible-errors}. 

The following result summarizes the discussion above and its consequences (proof in Appendix \ref{subsec-app:thm-score-error-principle-proof}).
\begin{theorem}
\label{thm:observable-score-error-principle}
Suppose \ref{ass:minimal-assumptions-marginals}--\ref{ass:fisher} hold, and let
$$
\be_s=\be_{\mathrm{obs},s}+\be_{\mathrm{inv},s},
\qquad
\be_{\mathrm{obs},s}:=\Pi_{\cGs}\be_s,
\quad
\be_{\mathrm{inv},s}:=\Pi_{\cGs^\perp}\be_s.
$$
Consider the backward Fokker-Planck equation,
\begin{equation}
\label{eq:thm-FP-generic}
\partial_s q_s
=
-\nabla\!\cdot\!\Big(\big[\boldsymbol{f}_s-\sigma_s^2\,\bs_\theta(\cdot, s)\big]q_s\Big)
\, - \, \frac{\sigma_s^2}{2}\Delta q_s
\, - \, \sigma_s^2\nabla\!\cdot\!\big(\bv_s\,q_s\big),
\qquad
q_T=\pstar_T,
\end{equation}
for $\bv_s=\be_s$ and for $\bv_s=\be_{\mathrm{obs},s}$.
Then:
\begin{enumerate}
      \item[\textnormal{(i)}] \emph{Only the observable component enters the marginal dynamics:} the target marginal curve $(\pstar_s)_s$ solves~\eqref{eq:thm-FP-generic} both for $\bv_s=\be_s$ and for $\bv_s=\be_{\mathrm{obs},s}$. In particular, it coincides with any solution for $\bv_s=\be_{\mathrm{obs},s}$ whenever that solution is unique.

    \item[\textnormal{(ii)}] \emph{Purely invisible errors:} if $\be_{\mathrm{obs},s}\equiv 0$ and $\be_{\mathrm{inv},s}\not\equiv 0$, then $\phat_s=\pstar_s$ for $s\in[t_0,T]$, while
    $$
    \int_{t_0}^T\|\be_s\|_{L^2(\pstar_s)}^2\,\md s
    =
    \int_{t_0}^T\|\be_{\mathrm{inv},s}\|_{L^2(\pstar_s)}^2\,\md s
    >0.
    $$
    
    \item[\textnormal{(iii)}] \emph{No lower bound:} Assume $d\ge2$. For any divergence $\mathrm{Div}(\cdot\|\cdot)$ between probability distributions such that $\mathrm{Div}(\mu\|\nu)=0$ iff $\mu=\nu$, there is no lower bound of the form
    $$
    \mathrm{Div}(\pstar_{t_0}\|\phat_{t_0})
    \ \ge\
    F\!\left(\int_{t_0}^T \,\|\be_s\|_{L^2(p^\star_s)}^2\,\md s\right)
    $$
    that holds uniformly over all score errors $\be$, for any strictly increasing function $F:[0,\infty)\to\R$ with $F(0)=0$. 
\end{enumerate}
\end{theorem}

\begin{figure}[t]
\centering
\begin{minipage}[c]{0.52\linewidth}
\centering
\small
\setlength{\tabcolsep}{4pt}
\begin{tabular}{l c c c}
\toprule
Capacity & Params & $\rho(\mathrm{FID},\,\hat{\mathcal{E}}_{\mathrm{full}})$ & $\rho(\mathrm{FID},\,\hat{\mathcal{E}}_{\mathrm{grad}})$ \\
\midrule
\textit{tiny}  & $1.76$M & $0.76 \pm 0.14$ & $\mathbf{0.97 \pm 0.03}$ \\
\textit{small} & $3.92$M & $0.89 \pm 0.01$ & $\mathbf{0.96 \pm 0.02}$ \\
\textit{full}  & $6.27$M & $0.94 \pm 0.03$ & $\mathbf{0.97 \pm 0.02}$ \\
\bottomrule
\end{tabular}
\end{minipage}
\hfill
\begin{minipage}[c]{0.46\linewidth}
\centering
\includegraphics[width=\linewidth]{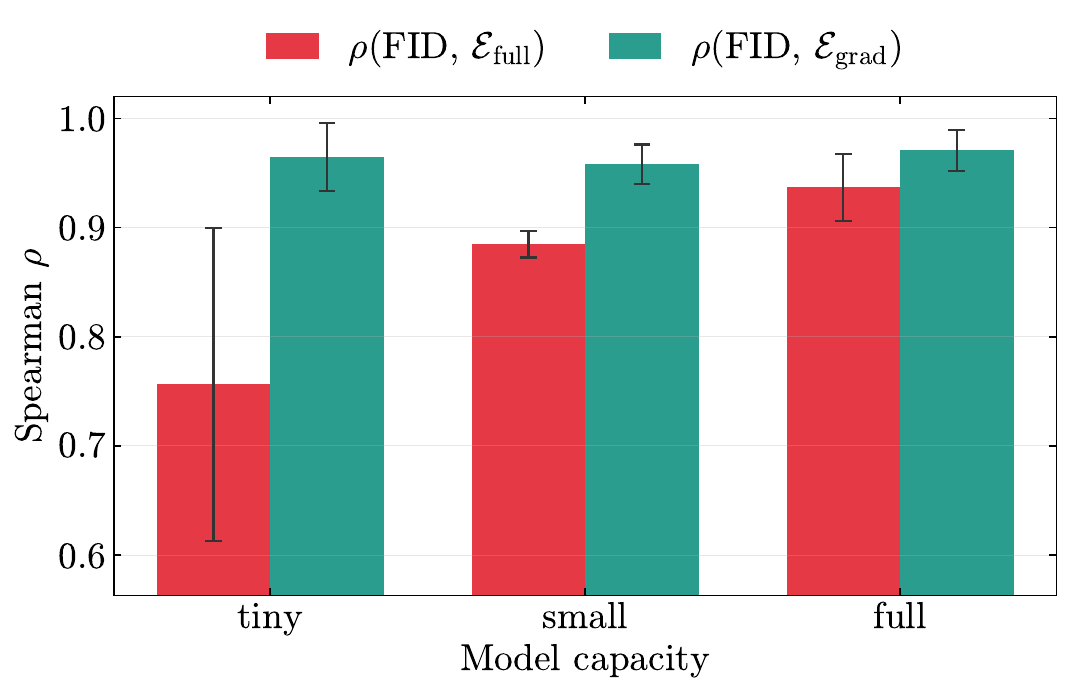}
\end{minipage}
\caption{\textbf{Spearman correlation between feature-FID \cite{Heuseletal2017} and score-error norms (Fashion-MNIST \cite{fashionMNIST_paper})}. Mean $\pm$ standard-deviation across three seeds. Across all capacities, the finite-critic estimate $\hat{\mathcal{E}}_{\mathrm{grad}}:=\widehat{\mathcal L}_{\mathrm{obs}}$ of $\E_{s,\bx_s}\!\left[\|\Pi_{\cG_s}\be_s\|_2^2\right]$ (Algorithm~\ref{alg:observable-error-critic}; a lower estimate at finite critic capacity) is uniformly more correlated with sample quality than the full error estimate $\hat{\mathcal{E}}_{\mathrm{full}} := \E_{s,\bx_s}\!\left[\|\be_s\|_2^2\right]$, with the largest gap at low capacity. Setup details in Appendix~\ref{app:two_phase_setup_fmnist}.}
\label{fig:spearman_fmnist}
\end{figure}

Theorem~\ref{thm:observable-score-error-principle} implies that the ambient $L^2$ score error is not the intrinsic quantity controlling sampling quality. Indeed the theorem shows that purely solenoidal errors may have arbitrarily large $L^2$ norm while leaving all marginals unchanged, and that no distributional discrepancy can be uniformly lower-bounded by the full score error alone. Notably, this implies that the classical bound
\begin{equation}
\label{eq:kl-upper-old}
\mathrm{KL}(\pstar_{t_0}\;\|\;\phat_{t_0}) 
\;\le\; 
\frac{1}{2}\int_{t_0}^T \sigma_s^2 \|\be_s\|_{L^2(\pstar_s)}^2\md s
\end{equation}
may be loose whenever score errors have large invisible (solenoidal) components.

By contrast, Theorem~\ref{thm:observable-score-error-principle} identifies the gradient projection of the error, $\Pi_{\cGs}\be_s$, as the component entering the Fokker-Planck dynamics at time $s$. Empirically, Figure~\ref{fig:spearman_fmnist} shows that the critic estimate of $\|\Pi_{\mathcal{G}}\mathbf{e}\|^2$ is a better proxy of sampling quality than the full error $\|\mathbf{e}\|^2$ as it correlates better with the feature Fréchet Inception Distance (FID) \cite{Heuseletal2017}, widely used to measure sampling quality of generative models for images. (The code for all the experiments reported here is available on \href{https://github.com/NailKhelifa/diffusion_models_observe_only_gradients}{GitHub}.) 

\paragraph{Improved upper bound.} Based on Theorem \ref{thm:observable-score-error-principle}, we replace the ambient $L^2$ geometry by an observable geometry, to obtain the following sharper endpoint upper bound (proof in Appendix \ref{subsec-app:kl-upper-bound-proof}).

\begin{theorem}
\label{thm:kl-upper-new}
Under Assumptions~\ref{ass:minimal-assumptions-marginals}--\ref{ass:fisher},
\begin{equation}
\label{eq:kl-upper-new}
\mathrm{KL}(\pstar_{t_0}\;\|\;\phat_{t_0})
\;\le\; 
\frac{1}{2}\int_{t_0}^T \sigma_s^2 \|\Pi_{\cGs}\be_s\|_{L^2(\pstar_s)}^2\md s.
\end{equation}
\end{theorem}

Theorem~\ref{thm:kl-upper-new} identifies the component of the error field that controls the endpoint divergence to be the $L^2(\pstar_s)$-projection of the error onto the subspace of gradient fields.

\begin{figure}[t]
    \centering
    \includegraphics[width=\linewidth]{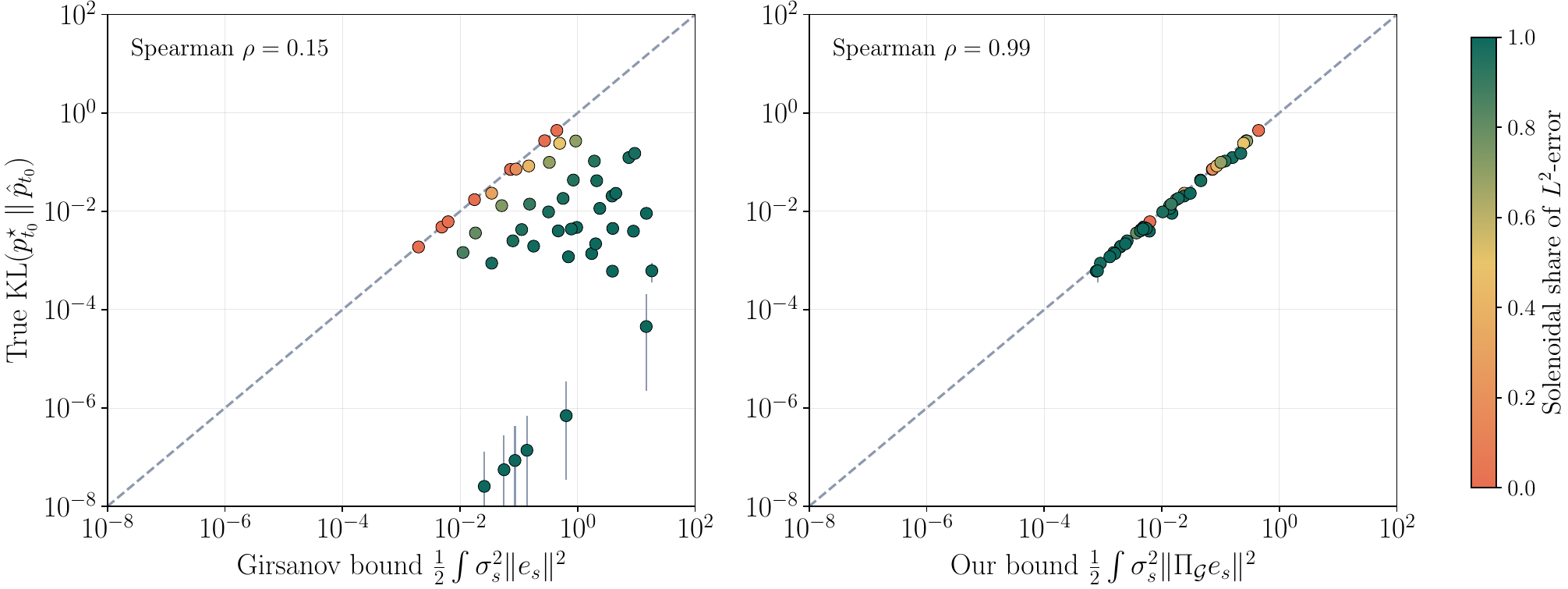}
    \caption{\textbf{Exact low-dimensional comparison of  standard Girsanov bound with that of Theorem \ref{thm:kl-upper-new}.}
    The target is an 8-mode anisotropic Gaussian mixture in $\R^2$ chosen so that forward marginals and scores are available in closed form. Each point is a synthetic score error
    $\be_s = a\,\nabla\psi_s + b\,J\nabla\log\pstar_s$, where the components are chosen so that the Helmholtz--Hodge decomposition is known: $\Pi_{\cGs}\be_s=a\nabla\psi_s$ and $\Pi_{\cGs^\perp}\be_s=b\,J\nabla\log\pstar_s$.
    The true $\KL(\pstar_{t_0}\|\phat_{t_0})$ is computed by solving the learned reverse Fokker--Planck equation on a fine grid (error bars: discretisation error). Colour: solenoidal share of $L^2$-error.
    \emph{Left:}~The true KL has low Spearman correlation with the full-$L^2$ (Girsanov) bound, consistent with
    Theorem~\ref{thm:observable-score-error-principle}~(iii): purely solenoidal errors ($a=0$, dark green dots) are taken into account by the Girsanov bound, while they do not affect the true KL (solver error $\approx 10^{-5}$).
    \emph{Right:}~Every point lies below the diagonal within numerical error, as guaranteed by Theorem~\ref{thm:kl-upper-new}, and the true KL is almost identical to our bound (Spearman $\rho=0.99$).
    Setup in Appendix~\ref{app:lowdim}.}
    \label{fig:lowdim-exact}
\end{figure}

\paragraph{Remarks}

   1.  \textit{Link with Girsanov-based approach.} In contrast with the Girsanov-based approach, this new bound is a direct consequence of the corrected Fokker-Planck dynamics \eqref{eq:FP-HH-learned-reverse} (discussed in Appendix \ref{subsec-app:kl-upper-bound-proof}). Appendix \ref{sec:girsanov-theory-link} reconciles the two by deriving the same bound via Girsanov's theorem applied to a \textit{marginally-equivalent representative} of the learned process whose drift retains only the observable component. This identifies the standard bound's looseness as the cost of operating on path-space rather than marginal-space dynamics.
   
    2. \textit{Conservative parametrizations.} If the learned score vector field is parametrized as $\bs_\theta(\bx, s) = \nabla_\bx E_\theta(\bx,s)$ for a scalar potential $E_\theta$ (e.g. as in energy-based models \cite{ebms}), then $\be_s = \nabla(E_\theta - \log \pstar_s)$ is a pure gradient field without any solenoidal component. In that case, the standard Girsanov bound coincides with~\eqref{eq:kl-upper-new}. For unconstrained architectures (e.g.\ U-Nets \cite{unet-seminal, DDPM} or Transformers \cite{vaswani2017attention, DiT}), which dominate practice, the gap between $\|\be_s\|_{L^2(\pstar_s)}$ and $\|\Pi_{\cGs}\be_s\|_{L^2(\pstar_s)}$ can be substantial and our new bound is tighter and correlates more strongly with the Fréchet Inception Distance, as shown in Figure \ref{fig:spearman_fmnist}.
   
    3. \textit{Gain with new upper bound.} The new upper bound \eqref{eq:kl-upper-new} is expected to significantly improve on the standard one \eqref{eq:kl-upper-old} when the gradient components of the errors have small norms compared to the solenoidal components. Figure \ref{fig:invisible-errors} illustrates this: by artificially creating solenoidal errors, our upper bound is not affected while that given in \eqref{eq:kl-upper-old} explodes.

\begin{figure}[t]
    \centering
    \includegraphics[width=\linewidth]{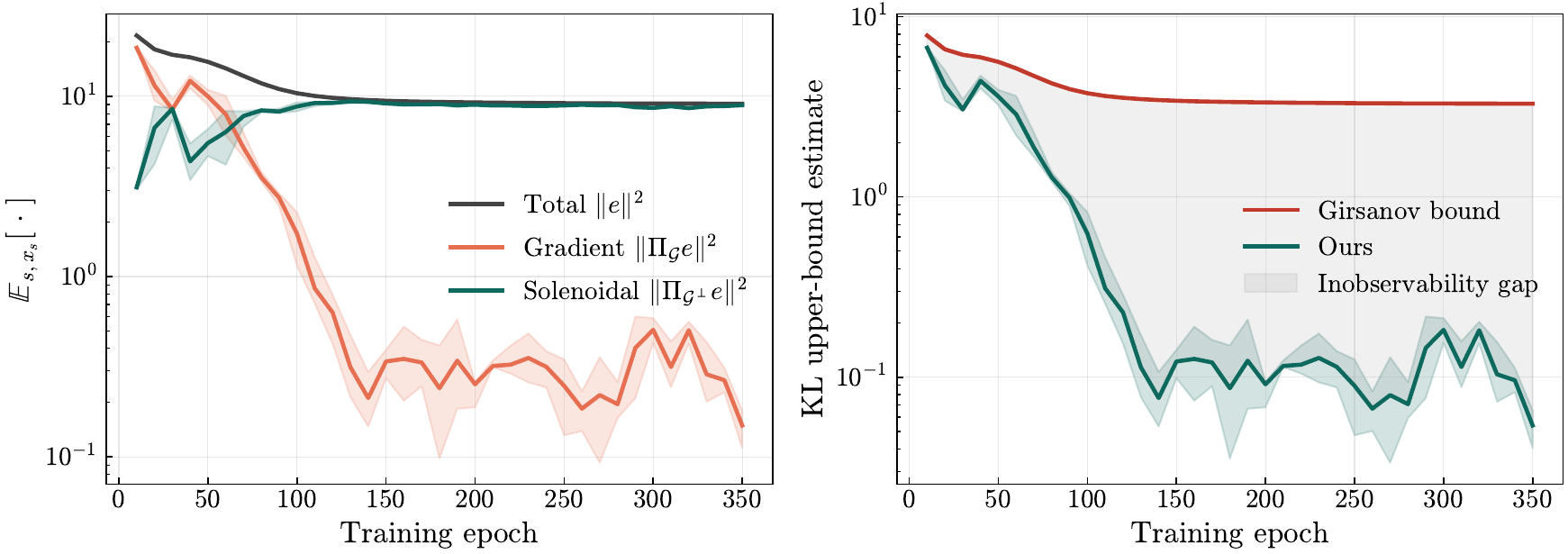}
    \caption{\textbf{DSM does not reduce the invisible component of the score error on CIFAR-10 \cite{cifar10}.}
    \emph{Left.} The gradient (observable) component of score estimation errors falls by roughly two orders of magnitude during training, while the solenoidal (invisible) component rises early and then plateaus, eventually accounting for nearly all of the residual error. \emph{Right.} As a consequence, the standard Girsanov bound~\eqref{eq:kl-upper-old} saturates at the solenoidal floor while our bound (Theorem~\ref{thm:kl-upper-new}) drops by two orders of magnitude and stays below it. The shaded \emph{inobservability gap} is what the standard analysis misses by penalizing error components invisible to marginals (Theorem~\ref{thm:observable-score-error-principle}). The gradient component is the finite-critic estimate $\widehat{\mathcal L}_{\mathrm{obs}}$ (a lower estimate), and the solenoidal component is obtained as the full error minus this estimate (hence an upper estimate). Median (solid) and inter-quartile range (band) over four seeds. Setup details in Appendix~\ref{app:two_phase_setup}. The same observations hold for Fashion-MNIST (Figures \ref{fig:kl-bounds-fmnist},  \ref{fig:error-decomp-fmnist}, Appendix \ref{sec:additional-results-fmnist}).}
    \label{fig:two-phase-and-bounds-cifar10}
\end{figure}

Figure~\ref{fig:lowdim-exact} makes this quantitative in a setting where the decomposition and the endpoint KL are known exactly and Figure~\ref{fig:two-phase-and-bounds-cifar10} illustrates this effect during training: after an initial phase in which the observable gradient component is reduced, the remaining score error is dominated by the solenoidal component, creating a persistent gap between the classical full-error bound and the observable bound.

\section{From Observable Bounds to a Computable Diagnostic}
\label{sec:from-observable-to-computable}

This section introduces a practical procedure to estimate the gradient component of score estimation errors. Theorem~\ref{thm:kl-upper-new} bounds the divergence between the target and learned data distributions using the $L^2(\pstar_s)$-gradient projection $\|\Pi_{\cG_s}\be_s\|_{L^2(\pstar_s)}$. Computing expectations with respect to $\pstar_s$ is practical since one has access to samples from the target distribution during training. However, explicitly projecting errors over the space of gradient fields is computationally hard. In this section, we present an approach that avoids explicit projections by characterizing the gradient component of $\be_s$ through an $H^{-1}$ dual variational identity.

\subsection{A Variational Reformulation}
\label{subsec:replacement-principle}

Recall from the Fokker-Planck equation \eqref{eq:FP-true-reverse-v2} that at each time $s$, the dynamics of the true densities depend on $\be_s$ only through $\nabla\!\cdot(\pstar_s\,\be_s)$. Thus any vector field $\bv_s$ such that $\nabla\!\cdot(\pstar_s\,\be_s) = \nabla \cdot (\pstar_s\,\bv_s)$ induces the same effect on marginals; we refer to such fields as admissible. In Section \ref{sec:FP-invisible-errors}, we observed that $\nabla\!\cdot(\pstar_s\be_s) = \nabla\!\cdot(\pstar_s\Pi_{\cGs}\be_s)$,
making the gradient-projection of errors one such admissible vector field, i.e. $\Pi_{\cGs}\be_s \in \{\bv_s: \nabla\!\cdot(\pstar_s\be_s) = \nabla\!\cdot(\pstar_s \bv_s)\}$. Moreover, since $\Pi_{\cGs}\be_s$ is the only observable part of $\be_s$, it contains the minimal and irreducible energy of $\be_s$ that is visible through $\nabla\!\cdot(\pstar_s\be_s)$. 

\begin{proposition}
    For each $s \in [t_0, T]$, 
    \begin{equation}
    \label{eq:grad-error-minimizes-L2}
    \|\Pi_{\cGs}\be_s\|_{L^2(\pstar_s)}^2
    =
    \inf_{\bv_s:\, \nabla\cdot(\pstar_s\bv_s)=\nabla\cdot(\pstar_s\be_s)}
    \int_{\mathbb R^d} \|\bv_s(\bx)\|_2^2\,\pstar_s(\bx)\,\md\bx.
    \end{equation}
    \label{prop:grad-error-minimizes-L2}
\end{proposition}
\begin{proof}
Consider any admissible field $\bv_s$, i.e. such that $\nabla\cdot(\pstar_s\bv_s)=\nabla\cdot(\pstar_s\be_s)$. Then, $\nabla\!\cdot\big(\pstar_s(\bv_s-\Pi_{\cGs}\be_s)\big)=0$ hence $\bv_s-\Pi_{\cGs}\be_s\in \cGs^\perp$. By orthogonality, $\|\bv_s\|_{L^2(\pstar_s)}^2 = \|\Pi_{\cGs}\be_s\|_{L^2(\pstar_s)}^2
+ \|\bv_s-\Pi_{\cGs}\be_s\|_{L^2(\pstar_s)}^2 \ge \|\Pi_{\cGs}\be_s\|_{L^2(\pstar_s)}^2.$ Thus, $\|\Pi_{\cGs}\be_s\|_{L^2(\pstar_s)}^2 = \inf_{\bv_s:\, \nabla\cdot(\pstar_s\bv_s) = \nabla\cdot(\pstar_s\be_s)} \|\bv_s\|_{L^2(\pstar_s)}^2.$
\end{proof}
In other words, among all vector fields $\bv_s$ that induce the weighted divergence $\nabla\cdot(\pstar_s\bv_s)=\nabla\cdot(\pstar_s\be_s)$, the field
$\Pi_{\cGs}\be_s$ is the unique minimum-energy representative in $L^2(\pstar_s;\R^d)$. Notably, the right-hand term of \eqref{eq:grad-error-minimizes-L2} does not require computing a projection, and just requires minimizing $L^2(\pstar_s)$-norms. However, this right-hand term is still not satisfactory since it requires solving a Poisson equation $\nabla\cdot(\pstar_s\bv_s)=\nabla\cdot(\pstar_s\be_s)$ which is known to be computationally intractable in high dimensions due to the curse of dimensionality inherent in standard numerical PDE solvers \cite{han2018solving}. Fortunately, the optimization problem in \eqref{eq:grad-error-minimizes-L2} exactly corresponds to an $H^{-1}$ norm \cite{Brezis2010-sobolev-pde}, which enjoys a tractable dual characterization. The $H^{-1}(\pstar_s)$-norm of $\nabla\cdot(\pstar_s\be_s)$ is defined as
\begin{equation}
\label{def:Hminus-1}
\|\nabla\cdot(\pstar_s\be_s)\|_{H^{-1}(\pstar_s)}^2
:=
\inf_{\bv_s:\, \nabla\cdot(\pstar_s\be_s)=\nabla\!\cdot(\pstar_s \bv_s)}
\int_{\mathbb R^d}\|\bv_s(\bx)\|^2_2\,\pstar_s(\bx)\,\md\bx,
\end{equation}
and can be expressed in its dual form as \cite{Brezis2010-sobolev-pde},
\begin{equation}
\label{eq:def-Hminus1-dual}
\|\nabla\cdot(\pstar_s\be_s)\|_{H^{-1}(\pstar_s)}^2
=
\sup_{\varphi\in C_c^\infty(\mathbb R^d)}
\left\{
-
2\int_{\mathbb R^d}\varphi\,\nabla\cdot(\pstar_s\be_s)\,\md\bx
-
\int_{\mathbb R^d}\|\nabla\varphi\|^2_2\,\pstar_s\,\md\bx
\right\}.
\end{equation}
From \eqref{eq:grad-error-minimizes-L2}-\eqref{eq:def-Hminus1-dual}, we have
\begin{equation}
\label{eq:final-approach}
\|\Pi_{\cGs}\be_s\|_{L^2(\pstar_s)}^2
=
\sup_{\varphi\in C_c^\infty(\mathbb R^d)}
\left\{
-
2\int_{\mathbb R^d}\varphi\,\nabla\cdot(\pstar_s\be_s)\,\md\bx
-
\int_{\mathbb R^d}\|\nabla\varphi\|^2_2\,\pstar_s\,\md\bx
\right\}.
\end{equation}

These $H^{-1}$-norms have a natural geometric interpretation in optimal transport theory (discussed in Appendix \ref{app:OT-interpretation}). Equation \eqref{eq:final-approach} is an unconstrained optimization problem over the set of smooth and scalar-valued test functions, and directly yields our method, described below, to estimate $\|\Pi_{\cGs}\be_s\|_{L^2(\pstar_s)}^2$.

\subsection{Estimation of the Observable Error}
\label{subsec:computational-aspect}

The  variational characterization  of 
$\|\Pi_{\cGs}\be_s\|^2_{L^2(\pstar_s)}$ in \eqref{eq:final-approach} 
suggests a clear training procedure to compute this norm: aim to solve the optimization by learning a scalar-valued critic network $\varphi_\psi$ with parameters $\psi$, via gradient ascent on the objective function in \eqref{eq:final-approach}. In the remainder of this section, we adopt the standard framework of variance-preserving forward diffusion \cite{song2021scorebased} (i.e. we fix $\boldsymbol{f}_s (\bx)= -\frac{1}{2} \beta(s)\bx$ and $\sigma_s = \sqrt{\beta(s)}$ in \eqref{eq:forward-sde} for some schedule $\beta:[t_0, T] \to \R_{>0}$), widely used in practical training of diffusion models. Applying integration by parts to the first integral in \eqref{eq:final-approach}, we obtain 
\begin{equation}
\label{eq:Hminus1_force_dual}
\|\Pi_{\cGs}\be_s\|_{L^2(\pstar_s)}^2
=
\sup_{\varphi\in C_c^\infty(\R^d)}
\left\{
2\,\E_{\bx\sim \pstar_s}\!\big[\nabla\varphi(\bx)\cdot \be_s(\bx)\big]
-
\E_{\bx\sim \pstar_s}\!\big[\|\nabla\varphi(\bx)\|^2_2\big]
\right\},
\end{equation}
which only requires expectations over $\pstar_s$, accessible during training by sampling noised data $\bX_s\sim \pstar_s$. We introduce a scalar-valued neural network $\varphi_\psi(\bx,s)$ with parameters $\psi$, and define
\begin{equation}
\label{eq:J_theta_psi_pop}
\mathcal J(\psi)
:=
\E_{s\sim \nu}\,
\E_{\bx\sim \pstar_s}
\Big[
2\,\nabla_\bx \varphi_\psi(\bx,s)\cdot \be_s(\bx)
-
\|\nabla_\bx \varphi_\psi(\bx,s)\|^2_2
\Big],
\end{equation}
where $\nu$ is a distribution over $[t_0, T]$. A finite critic class yields a \emph{lower} estimate of the observable energy, with equality for a sufficiently expressive class~\cite{Horniketal1989, lu2017expressive} under exact optimization. Indeed, if $\nabla\varphi\in\cGs$, then $\nabla\varphi\perp\Pi_{\cGs^\perp}\be_s$ in $L^2(\pstar_s)$, so that
$2\langle\nabla\varphi,\be_s\rangle-\|\nabla\varphi\|^2=\|\Pi_{\cGs}\be_s\|^2-\|\nabla\varphi-\Pi_{\cGs}\be_s\|^2$.
Hence $\mathcal J(\psi)\le\E_{s\sim\nu}\|\Pi_{\cGs}\be_s\|^2_{L^2(\pstar_s)}$, with equality if and only if $\nabla_\bx\varphi_\psi(\cdot,s)=\Pi_{\cGs}\be_s$ for $\nu$-a.e.\ $s$. 

To avoid the dependence of \eqref{eq:J_theta_psi_pop} on the unknown true score $\nabla_\bx \log \pstar$ through $\be_s=\bs_\theta(\cdot, s) -\nabla_\bx \log \pstar_s$, we exploit the DSM identity \cite{denoising-score-matching-vincent}: for $\alpha(s) = \exp(-\tfrac{1}{2}\int_{0}^s \beta(u)\md u)$ and $\sigma(s)^2 = 1 - \alpha(s)^2$, one has that  $\bX_s=\alpha(s)\bX_0+\sigma(s)\bepsilon$ with $\bX_0\sim \pstar$, and $\bepsilon\sim\cN(0,\bI_d)$, and
\begin{equation}
\label{eq:dsm_target_conditional}
\nabla_\bx \log \pstar_s(\bx)
=
\E\big[\bs_{\mathrm{target}}(\bX_s,s)\,\big|\,\bX_s=\bx\big],
\quad \text{ where } \quad  \bs_{\mathrm{target}}(\bX_s,s)
=
-\frac{\bX_s-\alpha(s)\bX_0}{\sigma(s)^2}.
\end{equation}
Consequently, the residual $\br_\theta(\bx,s) := \bs_\theta(\bx,s)-\bs_{\mathrm{target}}(\bx,s)$ is an unbiased proxy for the score error in the sense that $\E[\br_\theta(\bX_s,s)\mid\bX_s=\bx]=\be_s(\bx)$. Plugging this proxy in \eqref{eq:J_theta_psi_pop} yields, 
\begin{equation}
\label{eq:J_theta_psi_emp}
\widehat{\mathcal J}(\psi)
:=
\E_{s\sim \nu, \bX_0,\bepsilon}
\Big[
2\,\nabla_\bx \varphi_\psi(\bX_s,s)\cdot \br_\theta(\bX_s,s)
-
\|\nabla_\bx \varphi_\psi(\bX_s,s)\|^2
\Big],
\quad \bX_s=\alpha(s)\bX_0+\sigma(s)\bepsilon.
\end{equation}
The tower property gives $\E[\widehat{\mathcal J}(\psi)]=\mathcal J(\psi)$ for every \emph{fixed} $\psi$. We then train our critic network $\varphi_\psi$ to solve the following maximization problem,
\begin{equation}
\label{eq:forcing_loss_final}
\hat{\mathcal{L}}_{\mathrm{obs}}
:=
\sup_{\psi}\,\widehat{\mathcal J}(\psi).
\end{equation}
Intuitively, optimizing \eqref{eq:forcing_loss_final} makes the critic search for the \emph{steepest gradient direction} for the frozen score network (i.e., the direction that best exposes the gradient component of the score error). Algorithm~\ref{alg:observable-error-critic} in Appendix \ref{app:two_phase_setup} summarizes the resulting estimator.

\paragraph{Uses of this diagnostic.} We use $\widehat{\mathcal{L}}_{\mathrm{obs}}$ as a diagnostic rather than as a training objective, decoupling critic estimation from score-network optimization. The diagnostic can be applied in two ways. 

\emph{During training}: one may periodically refit the critic $\varphi_\psi$ against the frozen current score network $\hat{\bs}_\theta$, so that $\widehat{\mathcal L}_{\mathrm{obs}}$ tracks the gradient component of the score error along the optimization trajectory. 

\emph{Post hoc}: once fully trained, the score network can be frozen and the critic optimized more thoroughly to estimate the upper bound in Theorem~\ref{thm:kl-upper-new}. This is precisely the procedure that produces the gradient-component curves of Figures~\ref{fig:spearman_fmnist},~\ref{fig:two-phase-and-bounds-cifar10},~\ref{fig:kl-bounds-fmnist} and \ref{fig:error-decomp-fmnist}: at each checkpoint we freeze the score network, fit a fresh critic via~\eqref{eq:J_theta_psi_emp}, and report $\widehat{\mathcal L}_{\mathrm{obs}}$ as our estimate of the observable error (full setup, hyperparameters, and a critic-suboptimality ablation in Appendices~\ref{app:two_phase_setup_fmnist}--\ref{app:two_phase_setup}). A single evaluation of $\widehat{\mathcal L}_{\mathrm{obs}}$ takes a few minutes on one GPU and adds roughly $5$--$10\%$ overhead to standard DSM training, with cost driven by critic optimization, since the estimator only requires forward-noised samples $\bX_s = \alpha(s)\bX_0 + \sigma(s)\bepsilon$. Unlike sample-based metrics such as FID~\cite{Heuseletal2017}, the diagnostic $\widehat{\mathcal L}_{\mathrm{obs}}$ targets a quantitative upper bound on $\KL(\pstar_{t_0}\|\phat_{t_0})$ rather than perceptual quality and avoids reverse-SDE sampling entirely, making it complementary to FID and applicable in settings where sampling is expensive or where domain-specific feature extractors are unavailable.

\section{Discussion and Future Work}
\label{sec:discussion}

The bound of Theorem~\ref{thm:kl-upper-new} on $\mathrm{KL}(\pstar_{t_0}\;\|\;\phat_{t_0})$ integrates the instantaneous gradient component of the errors along the path, while the sampling distribution $\phat_{t_0}$ is shaped by how these gradient errors accumulate at the endpoint. We expect the diffusive part of the reverse SDE to have a smoothing effect, with some gradient errors injected at time $s$  dissipated before reaching $\phat_{t_0}$, and others amplified. Identifying which gradient components along the path affect the endpoint marginal is a fundamental open question, as pointed out in \cite{khelifa2026errorpropagationmodelcollapse}, and a sharper notion of endpoint observability building on our geometric framework would refine the bound further. 

Beyond this refinement, our decomposition admits a natural interpretation in Wasserstein-2 geometry: $\Pi_{\cGs}\be_s$ is  the component of the score error that lies in the tangent space at $\pstar_s$ of the Otto manifold; the solenoidal part lies in the orthogonal complement and does not move mass. In addition to providing geometric intuition, this perspective opens connections to gradient flows in Wasserstein space and to geometry-aware regularizers. Finally, our estimator of the norm of the gradient component of errors is currently used only as a training diagnostic. It freezes the score network to track gradient components during training. This could be extended into a training objective, either in combination with the standard DSM loss, or to make the critic guide the score network towards observable directions that translate into better learning and better capacity. We note that this raises subtle min-max optimization questions, since errors in estimating the critic could contaminate the score-network training signal if not carefully addressed.

\section*{Acknowledgements}
NBK is supported by a G-Research Trinity College Studentship, and RET is supported by the EPSRC Probabilistic AI Hub (EP/Y028783/1). RV was supported in part by an EPSRC Mathematical Sciences Small Grant.

\bibliographystyle{plain} 
\bibliography{reference}          

\newpage
\appendix

\section{Extended Related Work}
\label{app:related-work}

\paragraph{Score Decomposition.}
Horvat and Pfister~\cite{gauge-freedom-conservativity} used the Helmholtz-Hodge decomposition of the score field into a gradient and a divergence-free part to show that exact sampling and density estimation only depend on the gradient part, and illustrated this on low-dimensional Gaussian examples. Vuong et al.~\cite{vuong2025are} measured the non-conservativity of trained networks through loop integrals and reinterpreted diffusion training as matching a Wasserstein gradient flow. Chao et al.~\cite{on-investigating-conservative-property-sgm} measured non-conservativity through Jacobian-asymmetry statistics, reported that it can degrade sampling, and proposed a regularizer to address this. Richter-Powell et al.~\cite{richterpowell2022neural} propose divergence-free architectures based on differential forms, while Lai et al.~\cite{fp-diffusion} regularize the score network towards the score Fokker--Planck equation satisfied by the true family of scores. Salimans and Ho~\cite{should-ebm-model-the-energy-of-the-score} found that constraining the network to be a gradient costs little or nothing in sample quality, and Saremi \cite{saremi2019approximatingnablafneural} showed that unconstrained networks generically cannot represent exact gradients.

Our work differs from existing work in three ways. We (i) decompose the score \emph{error} rather than the score field, (ii) derive quantitative consequences in the approximate regime: an endpoint KL bound (Theorem~\ref{thm:kl-upper-new}) and an impossibility result (Theorem~\ref{thm:observable-score-error-principle}~(iii)), and (iii) estimate the observable error energy in trained image models. 

\paragraph{Minimum-energy velocities and Wasserstein geometry.}
In the Benamou--Brenier and Otto optimal transport frameworks~\cite{benamou-brenier, otto, ambrosioGradientFlowsMetric2008}, a curve of marginals is represented by many velocity fields, and the minimum-kinetic-energy one is the gradient field in the tangent space (Appendix~\ref{app:OT-interpretation}). Action Matching~\cite{action-matching} learns this gradient velocity from samples using a scalar neural potential and a variational objective. Our estimator in Section \ref{subsec:computational-aspect} uses the same tangent-space geometry, but applies it to the perturbation generated by the score error, $\nabla\cdot(\pstar_s\be_s)$, which involves the unknown true score.

\paragraph{Neural Helmholtz--Hodge decompositions and variational critics.}
Classical Helmholtz--Hodge methods~\cite{bhatia2013helmholtz} solve elliptic equations for scalar and vector potentials on grids. They are accurate in low dimension but suffer from the curse of dimensionality~\cite{han2018solving}.  Learned Stein discrepancies~\cite{learning-the-stein-discrepancy}, Sobolev GANs~\cite{sobolev-gan} and the Deep Ritz method~\cite{deep-ritz-method} optimize neural critics or potentials in variational formulations related to ours. None of these targets the quantity in our bound. The specific difficulty is that the field to be projected, $\be_s$, contains the unknown score $\nabla\log\pstar_s$. Because $\be_s$ enters the dual objective~\eqref{eq:Hminus1_force_dual} linearly, it can be replaced by the conditionally unbiased DSM residual. The resulting estimator only needs forward-noised data and the frozen score network.

\paragraph{Error weighting and manifold effects.}
Timestep and frequency reweighting schemes~\cite{minSNR, choiperception, wang2025an, beta2025sampling} and analyses of score matching near low-dimensional data manifolds~\cite{pidstrigach2022score, convergence-manifold-hypotheses-vdb} also reflect that not all score errors matter equally for sample quality. These works act on the full error across noise levels; the decomposition studied here is complementary to them and can be combined with any time weighting. 

\section{Proofs of Main Results}

\subsection{Preliminary results}
\label{subsec-app:cutoff}

\paragraph{Cutoff functions.} \label{par:cutoff} The proofs of the two following results (Lemma \ref{lem:cutoff} and Proposition \ref{prop:ibp-closure}) rely on cutoff functions, which we define as follows. Fix $\chi\in C_c^\infty(\R^d)$ with $0\le\chi\le1$, $\chi(\bx)=1$ for $\|\bx\|\le1$ and $\chi(\bx)=0$ for $\|\bx\|\ge2$. Let $C_\chi:=\sup_{\by\in\R^d}\|\nabla\chi(\by)\|<\infty$, and set $\chi_R:=\chi(\cdot/R)$ for $R\ge1$, so that $0\le\chi_R(\bx)\le1$ and $\chi_R(\bx)=1$ if $\|\bx\|\le R$. Since $\chi$ is constant on the open sets $\{\|\by\|<1\}$ and $\{\|\by\|>2\}$, $\nabla\chi(\by)\neq0$ only if $1\le\|\by\|\le2$. As $\nabla\chi_R(\bx)=R^{-1}\nabla\chi(\bx/R)$, for every $R\ge1$ and $\bx\in\R^d$,
$$
\|\nabla\chi_R(\bx)\|\le\frac{C_\chi}{R}\,\mathbf 1_{\{R\le\|\bx\|\le2R\}}.
$$
In particular $\chi_R\to1$ pointwise as $R\to\infty$. Moreover, if $R\le\|\bx\|\le2R$, then $1+\|\bx\|\le1+2R\le3R$ and $1+\|\bx\|^2\le1+4R^2\le5R^2$, so that, for all $\bx\in\R^d$,
\begin{equation}
\label{eq:cutoff-gradient-bounds}
\|\nabla\chi_R(\bx)\|\le\frac{3C_\chi}{1+\|\bx\|}\,\mathbf{1}_{\{R\le\|\bx\|\le2R\}},
\qquad
\|\nabla\chi_R(\bx)\|^2\,(1+\|\bx\|^2)\le5C_\chi^2\,\mathbf{1}_{\{R\le\|\bx\|\le2R\}}.
\end{equation}

\begin{lemma}
\label{lem:cutoff}
Let $p,g\in C^1(\R^d)$ satisfy
\begin{equation}
\label{eq:cutoff-ibp-hyp}
\int_{\R^d}\Big(|g|+(1+\|\bx\|)\,\|\nabla g\|\Big)\,|p|\,\md\bx<\infty,
\end{equation}
and let $\mathbf F:=p\,\bv$, where $\bv:\R^d\to\R^d$ is locally Lipschitz and satisfies $\|\bv(\bx)\|\le C_{\bv}(1+\|\bx\|)$ for all $\bx\in\R^d$. Then $\mathbf F\cdot\nabla g\in L^1(\R^d)$,i.e. is Lebesgue integrable. If, moreover, $g\,\nabla\!\cdot\mathbf F\in L^1(\R^d)$, then
\begin{equation}
\label{eq:cutoff-ibp}
\int_{\R^d}g\,\nabla\!\cdot\mathbf F\,\md\bx=-\int_{\R^d}\mathbf F\cdot\nabla g\,\md\bx .
\end{equation}
\end{lemma}

\begin{proof}
Since $p$ is $C^1$, on each ball $B_\rho:=\{\|\bx\|\le\rho\}$, $p$ is bounded by $M_p:=\sup_{B_\rho}|p|$ and, by the mean value theorem, Lipschitz with constant $L_p:=\sup_{B_\rho}\|\nabla p\|$, while $\bv$ is bounded by $C_{\bv}(1+\rho)$ and Lipschitz with some constant $L_{\bv}$. Hence, for all $\bx,\by\in B_\rho$,
\begin{align}
\label{eq:locally-lipschitz-F}
\|\bF(\bx)-\bF(\by)\|
&\le|p(\bx)|\,\|\bv(\bx)-\bv(\by)\|+\|\bv(\by)\|\,|p(\bx)-p(\by)|\nonumber \\ 
&\le\big(M_pL_{\bv}+C_{\bv}(1+\rho)L_p\big)\|\bx-\by\|.
\end{align}
As a consequence, $\bF=p\bv$ is locally Lipschitz and, by Rademacher's theorem~\cite[Section~3.1.2]{evans2015measure}, $\nabla\!\cdot\mathbf F$ exists almost everywhere. By the Cauchy--Schwarz inequality and the growth bound $\|\bv(\bx)\|\le C_{\bv}(1+\|\bx\|)$, for all $\bx\in\R^d$,
\begin{align*}
|\mathbf F\cdot\nabla g|
&=|p|\,|\bv\cdot\nabla g|
\le|p|\,\|\bv\|\,\|\nabla g\|
\le C_{\bv}\,(1+\|\bx\|)\,\|\nabla g\|\,|p|,
\end{align*}
and
\begin{align*}
\frac{|g|\,\|\bF\|}{1+\|\bx\|}
&=\frac{|g|\,|p|\,\|\bv\|}{1+\|\bx\|}
\le\frac{C_{\bv}(1+\|\bx\|)\,|g|\,|p|}{1+\|\bx\|}
=C_{\bv}\,|g|\,|p|,
\end{align*}
and both right-hand sides are integrable by~\eqref{eq:cutoff-ibp-hyp}. Hence $\bF\cdot\nabla g\in L^1(\R^d)$, which is the first claim, and $|g|\,\|\bF\|/(1+\|\cdot\|)\in L^1(\R^d)$.

Assume now that $g\,\nabla\!\cdot\bF\in L^1(\R^d)$, and fix $R\ge1$. The vector field $\bG_R:=\chi_R\,g\,\bF$ (for a cutoff function $\chi_R$ as described  in the paragraph above Lemma \ref{lem:cutoff}) is locally Lipschitz as the product of a locally Lipschitz function and a locally Lipschitz vector field by the same two-term estimate as for $\bF$ in \eqref{eq:locally-lipschitz-F}. Moreover, $\bG_R$ vanishes outside $\{\|\bx\|\le2R\}$ because $\chi_R$ does. Choose $\psi\in C_c^\infty(\R^d)$ with $\psi=1$ on $\{\|\bx\|\le2R+1\}$. By definition of the weak derivative \cite{evans2010pde}, and since $\nabla\psi=0$ wherever $\bG_R\neq0$,
$$
\int_{\R^d}\nabla\!\cdot\bG_R\,\md\bx
=\int_{\R^d}\psi\,\nabla\!\cdot\bG_R\,\md\bx
=-\int_{\R^d}\bG_R\cdot\nabla\psi\,\md\bx
=0 .
$$
Expanding $\nabla\!\cdot\bG_R$ with the product rule gives, for every $R\ge1$,
\begin{equation}
\label{eq:cutoff-identity}
0=\int_{\R^d}\chi_R\,\bF\cdot\nabla g\,\md\bx+\int_{\R^d}\chi_R\,g\,\nabla\!\cdot\mathbf F\,\md\bx+\int_{\R^d}g\,\bF\cdot\nabla\chi_R\,\md\bx .
\end{equation}
As $R\to\infty$, the first two integrals converge to $\int\mathbf F\cdot\nabla g\,\md\bx$ and $\int g\,\nabla\!\cdot\mathbf F\,\md\bx$ by dominated convergence, with the integrable dominating functions $|\mathbf F\cdot\nabla g|$ and $|g\,\nabla\!\cdot\mathbf F|$. For the third, by~\eqref{eq:cutoff-gradient-bounds},
$$
\Big|\int_{\R^d}g\,\mathbf F\cdot\nabla\chi_R\,\md\bx\Big|
\le3C_\chi\int_{\{\|\bx\|\ge R\}}\frac{|g(\bx)|\,\|\mathbf F(\bx)\|}{1+\|\bx\|}\,\md\bx
\xrightarrow[R\to\infty]{}0
$$
by dominated convergence. Letting $R\to\infty$ in~\eqref{eq:cutoff-identity} yields~\eqref{eq:cutoff-ibp}.
\end{proof}

We now prove a specialized result that applies Lemma \ref{lem:cutoff} in our context. 

\begin{proposition}
    \label{prop:ibp-closure}
    Under~\ref{ass:minimal-assumptions-marginals}--\ref{ass:fisher} and for $r_s:=\pstar_s/\phat_s$, the following hold for a.e. $s\in[t_0,T]$:
    \begin{enumerate}[label=(\roman*)]
    \item Both $(p,g)=(\pstar_s,\log r_s)$ and $(p,g)=(\phat_s,r_s)$ satisfy~\eqref{eq:cutoff-ibp-hyp}. Moreover, for each of these pairs and every $\bF=p\bv$ as in Lemma \ref{lem:cutoff} with $|\nabla\!\cdot\bF|=|\partial_s p|$ almost everywhere on $\R^d$, we have $g\,\nabla\!\cdot\bF\in L^1(\R^d)$, so that~\eqref{eq:cutoff-ibp} holds.
    \item $\nabla\log r_s\in\cGs$, where $\cGs$ is the closure in $L^2(\pstar_s;\R^d)$ of $\{\nabla\varphi:\varphi\in C_c^\infty(\R^d)\}$.
    \end{enumerate}
\end{proposition}

\begin{proof}
Fix $s\in[t_0,T]$ such that the linear-growth bound of~\ref{ass:coeff-reg}(iii) holds at $s$, the conclusions of~\ref{ass:fisher} hold at $s$, and
\begin{equation}
\label{eq:cutoff-good-times}
\int_{\R^d}\Big(|\partial_s\pstar_s|\,\big(1+|\log r_s|\big)+|\partial_s\phat_s|\,r_s\Big)\,\md\bx<\infty .
\end{equation}
The integrand in~\eqref{eq:cutoff-good-times} is nonnegative and, by~\ref{ass:density-reg}, its integral over $[t_0,T]\times\R^d$ is finite, so~\eqref{eq:cutoff-good-times} also holds for a.e.\ $s$ by Tonelli's theorem. We prove (i) and (ii) for such an $s$, which is fixed from now on.

\emph{Growth of the log-ratio.} By~\ref{ass:minimal-assumptions-marginals}, $\pstar_s$ and $\phat_s$ are strictly positive and of class $C^1$, so $\log r_s=\log\pstar_s-\log\phat_s$ and $r_s$ belong to $C^1(\R^d)$. By definition of $\be_s$, $\nabla\log\pstar_s=\bs_\theta(\cdot,s)-\be_s$, so~\ref{ass:coeff-reg}(iii) gives $\|\nabla\log\pstar_s(\bx)\|\le C(1+\|\bx\|)$, and~\ref{ass:fisher} gives $\|\nabla\log\phat_s(\bx)\|\le C_s(1+\|\bx\|)$ for some $C_s<\infty$. Hence $\|\nabla\log r_s(\bx)\|\le(C+C_s)(1+\|\bx\|)$. Applying the fundamental theorem of calculus to $t\mapsto\log r_s(t\bx)$ on $[0,1]$ and using $\|\bx\|\le\frac12(1+\|\bx\|^2)$,
\begin{align*}
|\log r_s(\bx)|
&\le|\log r_s(\mathbf 0)|+\int_0^1\|\nabla\log r_s(t\bx)\|\,\|\bx\|\,\md t
\le|\log r_s(\mathbf 0)|+(C+C_s)\Big(\|\bx\|+\tfrac12\|\bx\|^2\Big)\\
&\le|\log r_s(\mathbf 0)|+(C+C_s)(1+\|\bx\|^2).
\end{align*}
Thus, with $c_s:=|\log r_s(\mathbf 0)|+C+C_s$, for all $\bx\in\R^d$,
\begin{equation}
\label{eq:cutoff-growth}
\|\nabla\log r_s(\bx)\|\le c_s(1+\|\bx\|),
\qquad
|\log r_s(\bx)|\le c_s(1+\|\bx\|^2).
\end{equation}

\emph{Proof of (i).}
Both pairs consist of $C^1$ functions. We use~\eqref{eq:cutoff-growth}, $(1+\|\bx\|)^2\le2(1+\|\bx\|^2)$, and the finite second moment of $\pstar_s$ from~\ref{ass:minimal-assumptions-marginals}. For $(p,g)=(\pstar_s,\log r_s)$,
$$
\int_{\R^d}\Big(|\log r_s|+(1+\|\bx\|)\,\|\nabla\log r_s\|\Big)\,\pstar_s\,\md\bx
\le3c_s\int_{\R^d}(1+\|\bx\|^2)\,\pstar_s\,\md\bx<\infty .
$$
For $(p,g)=(\phat_s,r_s)$, we have $\nabla r_s=r_s\nabla\log r_s$ and $\phat_s r_s=\pstar_s$, so
\begin{align*}
\int_{\R^d}\Big(r_s+(1+\|\bx\|)\,\|\nabla r_s\|\Big)\,\phat_s\,\md\bx
&=\int_{\R^d}\Big(1+(1+\|\bx\|)\,\|\nabla\log r_s\|\Big)\,\pstar_s\,\md\bx\\
&\le1+2c_s\int_{\R^d}(1+\|\bx\|^2)\,\pstar_s\,\md\bx<\infty .
\end{align*}
Hence both pairs satisfy~\eqref{eq:cutoff-ibp-hyp}. Finally, if $|\nabla\!\cdot\mathbf F|=|\partial_s p|$ almost everywhere, then $|g\,\nabla\!\cdot\mathbf F|$ equals $|\partial_s\pstar_s|\,|\log r_s|$ for the first pair and $|\partial_s\phat_s|\,r_s$ for the second, both integrable by~\eqref{eq:cutoff-good-times}. Lemma \ref{lem:cutoff} then gives~\eqref{eq:cutoff-ibp}.

\emph{Proof of (ii).}
We first show that $\nabla \psi \in \cGs$ for $\psi \in C^1_c(\R^d)$. Let $h\in C_c^\infty(\R^d)$ with $h\ge0$, $h=0$ outside $\{\|\by\|\le1\}$ and $\int h=1$, and set $h_\varepsilon:=\varepsilon^{-d} h(\cdot/\varepsilon)$ for $\varepsilon\in(0,1]$. Then $\psi_\varepsilon:=h_\varepsilon*\psi\in C_c^\infty(\R^d)$ and $\nabla\psi_\varepsilon=h_\varepsilon*\nabla\psi$. Since $\nabla\psi$ is continuous with compact support, it is uniformly continuous; let $\omega$ be its modulus of continuity. Then, for every $\bx$,
$$
\|\nabla\psi_\varepsilon(\bx)-\nabla\psi(\bx)\|
\le\int_{\R^d}h_\varepsilon(\by)\,\|\nabla\psi(\bx-\by)-\nabla\psi(\bx)\|\,\md\by
\le\omega(\varepsilon),
$$
so, $\pstar_s$ being a probability density, $\|\nabla\psi_\varepsilon-\nabla\psi\|_{L^2(\pstar_s)}\le\omega(\varepsilon)\to0$ as $\varepsilon\to0$. Since each $\nabla\psi_\varepsilon$ lies in $\cGs$ and $\cGs$ is closed, $\nabla\psi\in\cGs$.

\medskip\noindent
Let $\varphi_R:=\chi_R\log r_s \in C^1_c(\R^d)$, where $\chi_R$ is a cutoff function as defined above Lemma \ref{lem:cutoff}. Then
$$
\nabla\varphi_R-\nabla\log r_s=(\chi_R-1)\,\nabla\log r_s+\log r_s\,\nabla\chi_R .
$$
Since $|\chi_R-1|\le\mathbf 1_{\{\|\bx\|>R\}}$ and $\nabla\log r_s=\nabla\log\pstar_s-\nabla\log\phat_s\in L^2(\pstar_s;\R^d)$ by~\ref{ass:fisher}, dominated convergence gives $\|(\chi_R-1)\nabla\log r_s\|_{L^2(\pstar_s)}\to0$. For the second term,~\eqref{eq:cutoff-growth} and~\eqref{eq:cutoff-gradient-bounds} give
\begin{align*}
\int_{\R^d}|\log r_s|^2\,\|\nabla\chi_R\|^2\,\pstar_s\,\md\bx
&\le c_s^2\int_{\R^d}(1+\|\bx\|^2)\,\big[\|\nabla\chi_R(\bx)\|^2(1+\|\bx\|^2)\big]\,\pstar_s(\bx)\,\md\bx\\
&\le5C_\chi^2c_s^2\int_{\{\|\bx\|\ge R\}}(1+\|\bx\|^2)\,\pstar_s(\bx)\,\md\bx,
\end{align*}
which tends to $0$ as $R\to\infty$ because $\pstar_s$ has a finite second moment by~\ref{ass:minimal-assumptions-marginals}. By the triangle inequality, $\|\nabla\varphi_R-\nabla\log r_s\|_{L^2(\pstar_s)}\to0$. Since $\varphi_R \in C^1_c(\R^d)$ by the above, $\nabla\varphi_R \in \cGs$. By closedness of $\cGs$, this implies $\nabla\log r_s\in\cGs$.
\end{proof}

\subsection{Proof of Theorem \ref{thm:observable-score-error-principle}}
\label{subsec-app:thm-score-error-principle-proof}

\begin{proof}
\noindent \textit{Proof of (i)}:   
Recall that the score estimation error field $\be_s$ is decomposed as: \begin{equation}
\label{eq:errors-decomp-proof}
\be_s=\be_{\mathrm{obs},s}+\be_{\mathrm{inv},s},
\qquad
\be_{\mathrm{obs},s}:=\Pi_{\cGs}\be_s,
\quad
\be_{\mathrm{inv},s}:=\Pi_{\cGs^\perp}\be_s,
\end{equation}
where $\cGs^\perp$ is the orthogonal complement in $L^2(p_s^\star;\mathbb R^d)$. Also recall from \eqref{eq:FP-true-reverse-v2} that the marginals of the true backward SDE $(\pstar_s)_{s \in [t_0, T]}$ satisfy, 
\begin{equation}
\label{eq:thm-FP-total}
\partial_s \pstar_s
=
-\nabla\!\cdot\!\Big(\big[\boldsymbol{f}_s-\sigma_s^2 \bs_\theta(\cdot, s)\big]\pstar_s\Big)
-\frac{\sigma_s^2}{2}\Delta \pstar_s
-\sigma_s^2\nabla\!\cdot\!\big(\pstar_s\,\be_s\big).
\end{equation}
Applying the decomposition \eqref{eq:errors-decomp-proof}, \eqref{eq:thm-FP-total} can be written as
\begin{equation}
\label{eq:thm-FP-total-obs-only}
\partial_s \pstar_s
=
-\nabla\!\cdot\!\Big(\big[\boldsymbol{f}_s-\sigma_s^2\bs_\theta(\cdot, s)\big]\pstar_s\Big)
-\frac{\sigma_s^2}{2}\Delta \pstar_s
-\sigma_s^2\nabla\!\cdot\!\big(\pstar_s\,\be_{\mathrm{obs}, s}\big),
\end{equation}
since, by definition of $\cGs$ (see \eqref{eq:divfree_cond}), $\nabla \cdot (\pstar_s \,\be_{\mathrm{inv},s})=0$.   Therefore $(\pstar_s)_{s\in[t_0,T]}$ solves~\eqref{eq:thm-FP-generic} with $\bv_s=\be_{\mathrm{obs},s}$ and terminal condition $\pstar_T$. This proves (i).

\medskip
\noindent
\textit{Proof of (ii)}: Assume $\be_{\mathrm{obs}}\equiv 0$. Then \eqref{eq:thm-FP-total}, or equivalently \eqref{eq:thm-FP-total-obs-only}, reduces to the following PDE, backward in time $s$, from $T$ to $t_0$:
\begin{equation} 
\label{eq:FP-obs-zero}
\partial_s \pstar_s
=
- \nabla\!\cdot\!\Big(\big[\boldsymbol{f}_s- \sigma_s^2\bs_\theta(\cdot, s)\big]\pstar_s\Big)
-\frac{\sigma_s^2}{2}\Delta \pstar_s.
\end{equation}
Writing $\nabla_\bx \log \pstar_s = \bs_\theta(\cdot, s) -\be_s$ in~\eqref{eq:FP-learned-reverse}, the $\be_s$ terms cancel, and $(\phat_s)_s$ solves the same equation~\eqref{eq:FP-obs-zero} with the same terminal condition $\pstar_T$. Since the drift $\boldsymbol f_s-\sigma_s^2\bs_\theta(\cdot,s)$ is locally Lipschitz with linear growth by~\ref{ass:coeff-reg}, this equation has a unique solution~\cite{bogachev_krylov_rockner_shaposhnikov}.

Thus, $\pstar_s=\phat_s$ for all $s\in[t_0,T]$, meaning that any divergence between these two marginals is $0$. Moreover, if $\be_{\mathrm{obs}}\equiv 0$, then $\be=\be_{\mathrm{inv}}$, hence
$$
\int_{t_0}^T \|\be_s\|_{L^2(\pstar_s)}^2\,\md s
=
\int_{t_0}^T \|\be_{\mathrm{inv},s}\|_{L^2(\pstar_s)}^2\,\md s.
$$
If $\be_{\mathrm{inv}}\not\equiv 0$, the right-hand side is strictly positive. This proves (ii).

\medskip
\noindent
\textit{Proof of (iii): no lower bound from the full $L^2$ score error.}

\medskip
Assume by contradiction that there exists a strictly increasing function
$F:[0,\infty)\to\R$ with $F(0)=0$ such that, for every score error field $\be$,
\begin{equation}
\label{eq:contradiction-lower-bound-full-L2}
\mathrm{Div}(\pstar_{t_0}\|\phat_{t_0})
\;\ge\;
F\!\left(\int_{t_0}^T \|\be_s\|_{L^2(\pstar_s)}^2\,\md s\right).
\end{equation}

We first exhibit an admissible purely invisible error. Let $J\in\R^{d\times d}$ be a nonzero antisymmetric matrix (this requires $d\ge2$), let $\psi\in C_c^\infty(\R^d)$ be non-constant, and for $\lambda>0$ set
$$
\bu^\lambda_s:=\lambda\,\frac{J\nabla\psi}{\pstar_s}.
$$
Then $\nabla\cdot(\pstar_s\bu^\lambda_s)=\lambda\,\nabla\cdot(J\nabla\psi)=\lambda\sum_{i,j}J_{ij}\partial_i\partial_j\psi=0$ by antisymmetry, so $\bu^\lambda_s\in\cG_s^\perp$. Moreover $\bu^\lambda_s$ is bounded, compactly supported, locally Lipschitz and jointly measurable, so the score $\bs_\theta(\cdot,s)=\nabla\log\pstar_s+\bu^\lambda_s$ satisfies~\ref{ass:minimal-assumptions-marginals}--\ref{ass:fisher} (this uses that $\nabla\log\pstar_s$ is locally Lipschitz with linear growth, which holds as soon as one score network satisfies~\ref{ass:coeff-reg}, since $\nabla\log\pstar_s=\bs_\theta-\be_s$), with $\phat_s=\pstar_s$ by part~(ii). Define $\be_s:=\bu^\lambda_s$ for some $\lambda>0$. Since $\bu_s^\lambda\in\cG_s^\perp$, we have for almost every $s\in[t_0,T]$ that $\Pi_{\cG_s}\be_s = 0$. Therefore, by part~(ii), the corresponding learned marginal curve coincides with the true one, and in particular $\phat_{t_0}=p_{t_0}^\star$. Hence, using that $\mathrm{Div}(\mu\|\nu)=0$ if and only if $\mu=\nu$, $\mathrm{Div}(\pstar_{t_0}\|\phat_{t_0})=0$. On the other hand,
$$
\int_{t_0}^T \|\be_s\|_{L^2(\pstar_s)}^2\,\md s
=
\int_{t_0}^T\|\bu^\lambda_s\|^2_{L^2(\pstar_s)}\md s
=\lambda^2\int_{t_0}^T\|\bu^1_s\|^2_{L^2(\pstar_s)}\md s
\underset{\lambda \to \infty}{\longrightarrow} + \infty
$$
Since $F$ is strictly increasing such that $F(0)=0$, one can pick $\lambda>0$ such that $F(\int_{t_0}^T\|\bu^\lambda_s\|^2_{L^2(\pstar_s)}\md s) >0$. But, applying \eqref{eq:contradiction-lower-bound-full-L2} to the error field $\be$ yields
$$
0
=
\mathrm{Div}(\pstar_{t_0}\|\phat_{t_0})
\;\ge\;
F(\int_{t_0}^T\|\bu^\lambda_s\|^2_{L^2(\pstar_s)}\md s)
>
0.
$$
This contradicts the previous inequality and proves that no such lower bound of the form \eqref{eq:contradiction-lower-bound-full-L2} can hold uniformly over all score error fields $\be$.
\end{proof}

\subsection{Proof of Theorem \ref{thm:kl-upper-new}}
\label{subsec-app:kl-upper-bound-proof}

We first state a useful lemma, which is similar to \cite[Lemma 6]{chen2023improved}, but we give a proof for completeness. 

\begin{lemma}
    Assume~\ref{ass:minimal-assumptions-marginals}--\ref{ass:fisher}. For  $s \in [t_0, T]$ let $r_s := \pstar_s/\phat_s$ denote the time-$s$  ratio of marginals. Then, for a.e. $s \in [t_0, T]$, 
    \begin{equation}
          -\frac{\md}{\md s}\KL(\pstar_s\,\|\,\phat_s)
          =
          -\frac{\sigma_s^2}{2}\,\cI(\pstar_s\,\|\,\phat_s)
          -
          \sigma_s^2 \int_{\R^d} (\be_s \cdot \nabla \log r_s) \, \pstar_s\,\md\bx,
    \end{equation}  
    \label{lem:derivative-kl-fisher}
 where $\cI(\pstar_s\|\phat_s):=\int_{\R^d}\|\nabla\log r_s\|^2\,\pstar_s\,\md\bx$ is the relative Fisher information of $\pstar_s$ with respect to $\phat_s$. The left-hand side is the rate of change of the divergence along the reverse dynamics, i.e.\ as $s$ decreases from $T$ to $t_0$.
\end{lemma}

\begin{proof}
By~\ref{ass:density-reg}, $s\mapsto\KL(\pstar_s\|\phat_s)$ is absolutely continuous and may be differentiated under the integral sign. In the proof of this lemma, we reparametrize reverse time by $\tau=T-s$. Thus $\tau\in[0,T-t_0]$, and all quantities below are understood as
$$
p_\tau^\star := p_{T-\tau}^\star,\qquad \hat p_\tau := \hat p_{T-\tau},\qquad \be_\tau := \be_{T-\tau},\qquad \boldsymbol{f}_\tau := \boldsymbol{f}_{T-\tau},\qquad \sigma_\tau := \sigma_{T-\tau}.
$$

\medskip
\noindent
We denote $r_\tau := \frac{\pstar_\tau}{\phat_\tau}$ and the common part of the reverse drift as
$$
\bb_\tau(\bx)
:=
-\boldsymbol{f}_\tau(\bx)+\sigma_\tau^2 \nabla \log \pstar_\tau(\bx).
$$
Then the true and learned backward Fokker-Planck equations read:
\begin{align}
\partial_\tau \pstar_\tau
&=
-\nabla\!\cdot(\bb_\tau \pstar_\tau)+\frac{\sigma_\tau^2}{2}\Delta \pstar_\tau,
\label{eq:pfp-p}
\\
\partial_\tau \phat_\tau
&=
-\nabla\!\cdot(\bb_\tau \phat_\tau)
+\frac{\sigma_\tau^2}{2}\Delta \phat_\tau
-\sigma_\tau^2 \nabla\!\cdot(\phat_\tau \be_\tau).
\label{eq:pfp-q}
\end{align}
The relative entropy between $\phat_\tau$ and $\pstar_\tau$ (which are assumed absolutely continuous relative to each other) is defined as:
\begin{align*}
\KL(\pstar_\tau\,\|\,\phat_\tau)
=
\int_{\R^d} \pstar_\tau\,(\bx)\log\!\frac{\pstar_\tau(\bx)}{\phat_\tau(\bx)}\,\md\bx
=
\int_{\R^d} \pstar_\tau\, \log r_\tau \,\md\bx.
\end{align*}

\medskip
\textit{Step 1: Differentiation of the relative entropy.}

\medskip
Using the product rule,
\begin{align*}
\frac{\md}{\md \tau}\KL(\pstar_\tau\,\|\,\phat_\tau)
=
\int_{\R^d} \partial_\tau \pstar_\tau\, \log r_\tau\,\md\bx
+
\int_{\R^d} \pstar_\tau\, \,\partial_\tau(\log r_\tau)\,\md\bx.
\end{align*}
Since $\partial_\tau(\log r_\tau)=\partial_\tau\pstar_\tau\,/\pstar_\tau\,-\partial_\tau \phat_\tau/\phat_\tau$, we obtain
\begin{align*}
\int_{\R^d} \pstar_\tau\, \,\partial_\tau(\log r_\tau)\,\md\bx
=
\int_{\R^d}\partial_\tau \pstar_\tau\,\md\bx
-
\int_{\R^d} \pstar_\tau\, \frac{\partial_\tau \phat_\tau}{\phat_\tau}\,\md\bx.
\end{align*}
Because $\pstar_\tau$ is a probability density for every $\tau$, $\int \partial_\tau \pstar_\tau\,\,\md\bx=0$,
hence
\begin{align}
\label{eq:dkl-basic}
\frac{\md}{\md \tau}\KL(\pstar_\tau\,\|\,\phat_\tau)
&= 
\int_{\R^d} \partial_\tau \pstar_\tau\, \log r_\tau\,\md\bx
-
\int_{\R^d} \pstar_\tau\, \frac{\partial_\tau \phat_\tau}{\phat_\tau}\,\md\bx \nonumber \\
&=
\int_{\R^d} \partial_\tau \pstar_\tau\, \log r_\tau\,\md\bx
-
\int_{\R^d} r_\tau \,\partial_\tau \phat_\tau\,\md\bx.
\end{align}

\medskip
\textit{Step 2: Justifying integration by parts.}

Let $\bF^\star_\tau:=\pstar_\tau\big(\bb_\tau-\tfrac{\sigma_\tau^2}{2}\nabla\log\pstar_\tau\big)$ and $\hat\bF_\tau:=\phat_\tau\big(\bb_\tau-\tfrac{\sigma_\tau^2}{2}\nabla\log\phat_\tau+\sigma_\tau^2\be_\tau\big)$, so that \eqref{eq:pfp-p}--\eqref{eq:pfp-q} read $\partial_\tau\pstar_\tau=-\nabla\!\cdot\bF^\star_\tau$ and $\partial_\tau\phat_\tau=-\nabla\!\cdot\hat\bF_\tau$. Both are of the form $p\,\bv$ with $\bv$ locally Lipschitz of linear growth, by~\ref{ass:coeff-reg}, \ref{ass:density-reg} and~\ref{ass:fisher}. Applying Proposition~\ref{prop:ibp-closure}(i) with $(p,g)=(\pstar_\tau,\log r_\tau)$ and $(p,g)=(\phat_\tau,r_\tau)$ in~\eqref{eq:dkl-basic} gives, for a.e.\ $\tau$,
\begin{align}
\label{eq:KL_AsBs}
\frac{\md}{\md\tau}\KL(\pstar_\tau\,\|\,\phat_\tau)
&=
\underbrace{\int\bb_\tau\pstar_\tau\cdot\nabla\log r_\tau-\int\bb_\tau\phat_\tau\cdot\nabla r_\tau +\frac{\sigma_\tau^2}{2}\Big(-\int\nabla\pstar_\tau\cdot\nabla\log r_\tau+\int\nabla\phat_\tau\cdot\nabla r_\tau\Big)}_{=:A_\tau}
\nonumber \\
&\;\underbrace{-\,\sigma_\tau^2\int\phat_\tau\,\be_\tau\cdot\nabla r_\tau}_{=:B_\tau}.
\end{align}
Splitting is legitimate: each of the five integrands is of the form $(p\,\bv)\cdot\nabla g$ with $\bv\in\{\bb_\tau,\nabla\log p,\be_\tau\}$, hence integrable by Lemma~\ref{lem:cutoff}. We now compute $A_\tau$ and $B_\tau$ separately.

\medskip
\textit{Step 3: Contribution of $A_\tau$.}
Since $\nabla r_\tau=r_\tau\nabla\log r_\tau$ and $r_\tau\phat_\tau=\pstar_\tau$, we have $\bb_\tau\phat_\tau\cdot\nabla r_\tau=\bb_\tau\pstar_\tau\cdot\nabla\log r_\tau$, so the two transport integrals in $A_\tau$ cancel exactly, and
$$
A_\tau=\frac{\sigma_\tau^2}{2}\left(-\int\nabla\pstar_\tau\cdot\nabla\log r_\tau\,\md\bx+\int\nabla r_\tau\cdot\nabla\phat_\tau\,\md\bx\right).
$$
Using $\pstar_\tau=r_\tau\phat_\tau$, we have
$\nabla\pstar_\tau=r_\tau\nabla\phat_\tau+\phat_\tau\nabla r_\tau$, and
\begin{align*}
-\nabla\pstar_\tau\cdot\nabla\log r_\tau
&=
-r_\tau\nabla\phat_\tau\cdot\nabla\log r_\tau
-\phat_\tau\nabla r_\tau\cdot\nabla\log r_\tau \\
&=
-\nabla\phat_\tau\cdot\nabla r_\tau
-\phat_\tau\nabla r_\tau\cdot\nabla\log r_\tau \\
&=
-\nabla\phat_\tau\cdot\nabla r_\tau
-\pstar_\tau\nabla\log r_\tau\cdot\nabla\log r_\tau\\
&=
-\nabla\phat_\tau\cdot\nabla r_\tau
-\pstar_\tau\|\nabla\log r_\tau\|^2,
\end{align*}
where we used $\nabla\log r_\tau=\nabla r_\tau/r_\tau$ in the second step and $\phat_\tau\nabla r_\tau=\pstar_\tau\nabla\log r_\tau$ in the third. The first term cancels the second integral in $A_\tau$, leaving
$$
A_\tau
=
-\frac{\sigma_\tau^2}{2}\int_{\R^d}\pstar_\tau(\bx)\,\|\nabla\log r_\tau(\bx)\|^2\,\md\bx
=
-\frac{\sigma_\tau^2}{2}\,\cI(\pstar_\tau\,\|\,\phat_\tau),
$$
where $\cI(\pstar_\tau\|\phat_\tau):=\int_{\R^d}\|\nabla\log r_\tau\|^2\,\pstar_\tau\,\md\bx$ is the relative Fisher information of $\pstar_\tau$ with respect to $\phat_\tau$.

\medskip
\textit{Step 4: Contribution of $B_\tau$.}
Using $\phat_\tau\nabla r_\tau=\pstar_\tau\nabla\log r_\tau$ in the definition of $B_\tau$ in~\eqref{eq:KL_AsBs},
$$
B_\tau=-\sigma_\tau^2\int_{\R^d}\pstar_\tau(\bx)\,\be_\tau(\bx)\cdot\nabla\log r_\tau(\bx)\,\md\bx .
$$
Combining $A_\tau$ and $B_\tau$ in \eqref{eq:KL_AsBs},
\begin{equation}
\label{eq:variations-KL-tau}
\frac{\md}{\md \tau}\KL(\pstar_\tau\,\|\,\phat_\tau)
=
-\frac{\sigma_\tau^2}{2}\,\cI(\pstar_\tau\,\|\,\phat_\tau)
-
\sigma_\tau^2 \int_{\R^d} (\be_\tau\cdot\nabla\log r_\tau)\,\pstar_\tau\,\md\bx.
\end{equation}
Equation \eqref{eq:variations-KL-tau} is written in forward time for $\tau \in [0, T-t_0]$, with $\md \tau > 0$. Since $s=T-\tau$, we have $\frac{\md}{\md\tau}=-\frac{\md}{\md s}$, which yields the statement of the lemma.
\end{proof}

Finally, here is the proof of Theorem \ref{thm:kl-upper-new}.

\begin{proof}

\medskip
By Lemma \ref{lem:derivative-kl-fisher}, for a.e. $s \in [t_0, T]$ in the reverse-time diffusion, 
\begin{equation}
\label{eq:entropy-balance-final-v2}
-\frac{\md}{\md s}\KL(\pstar_s\,\|\,\phat_s)
=
-\frac{\sigma_s^2}{2}\,\cI(\pstar_s\,\|\,\phat_s)
-
\sigma_s^2 \int_{\R^d} (\be_s \cdot \nabla \log r_s) \, \pstar_s\,\md\bx,
\end{equation}
where $r_s = \pstar_s/\phat_s$. By Proposition~\ref{prop:ibp-closure}(ii), $\nabla\log r_s$ belongs to the gradient subspace $\cGs$ (defined in~\eqref{eq:divfree_cond}). Using the Helmholtz-Hodge decomposition of the errors in \eqref{eq:HH-decomposition-errors}:
$$
\be_s = \Pi_{\cGs}\be_s + \Pi_{\cGs^\perp}\be_s
\ \in L^2(\pstar_s;\R^d),
$$
we obtain, since $\Pi_{\cGs^\perp}\be_s \in \cGs^\perp$ and $\nabla\log r_s \in \cGs$, that $ \int_{\R^d}(\Pi_{\cGs^\perp}\be_s \cdot \nabla\log r_s)\, \pstar_s \md \bx = 0$. Thus,
$$
\int_{\R^d} (\be_s\cdot \nabla\log r_s)\,\pstar_s\,\md\bx
=
\int_{\R^d} (\Pi_{\cGs}\be_s \cdot \nabla\log r_s)\,\pstar_s\,\md\bx.
$$
Then, by Cauchy--Schwarz,
$$
-\int_{\R^d} (\be_s\cdot \nabla\log r_s)\,\pstar_s\,\md\bx
=-\int_{\R^d} (\Pi_{\cGs}\be_s\cdot \nabla\log r_s)\,\pstar_s\,\md\bx
\le
\|\Pi_{\cGs}\be_s\|_{L^2(\pstar_s)}\,
\sqrt{\cI(\pstar_s\,\|\,\phat_s)}.
$$
Substituting into \eqref{eq:entropy-balance-final-v2}, this yields
$$
-\frac{\md}{\md s}\KL(\pstar_s\,\|\,\phat_s)
\le
-\frac{\sigma_s^2}{2}\,\cI(\pstar_s\,\|\,\phat_s)
+
\sigma_s^2
\|\Pi_{\cGs}\be_s\|_{L^2(\pstar_s)}
\sqrt{\mathcal I(\pstar_s\,\|\,\phat_s)}.
$$

By Young's inequality $ab\le \frac12 a^2+\frac12 b^2$ applied with $a=\sqrt{\cI(\pstar_s\,\|\,\phat_s)}$ and $b=\|\Pi_{\cGs}\be_s\|_{L^2(\pstar_s)}$, one gets,
$$
\|\Pi_{\cGs}\be_s\|_{L^2(\pstar_s)}
\sqrt{\cI(\pstar_s\,\|\,\phat_s)}
\le
\frac12 \cI(\pstar_s\,\|\,\phat_s)
+
\frac12 \|\Pi_{\cGs}\be_s\|_{L^2(\pstar_s)}^2.
$$
Hence
\begin{equation}
\label{eq:pointwise-kl-differential-bound}
-\frac{\md}{\md s}\KL(\pstar_s\,\|\,\phat_s)
\le
\frac{\sigma_s^2}{2}\|\Pi_{\cGs}\be_s\|_{L^2(\pstar_s)}^2.
\end{equation}
Since $\pstar_T=\phat_T$, we have $\KL(\pstar_T\|\phat_T)=0$, so that $\KL(\pstar_{t_0}\|\phat_{t_0})=-\int_{t_0}^T\frac{\md}{\md s}\KL(\pstar_s\|\phat_s)\,\md s$ (absolute continuity by~\ref{ass:density-reg}). Integrating~\eqref{eq:pointwise-kl-differential-bound} over $[t_0,T]$ gives
$$
\KL(\pstar_{t_0}\|\phat_{t_0})
\le
\frac12 \int_{t_0}^T \sigma_s^2
\|\Pi_{\cGs}\be_s\|_{L^2(\pstar_s)}^2\,\md s,
$$
which is exactly \eqref{eq:kl-upper-new}.
\end{proof}

\section{Connection with Girsanov's Theory}
\label{sec:girsanov-theory-link}

In this section, we formally develop the link between our Helmholtz-Hodge approach and its implications for the standard Girsanov-based approach to bound $\KL(\pstar_{t_0}\|\phat_{t_0})$. We prove that the upper bound in Theorem \ref{thm:kl-upper-new} can also be established (with an additional assumption) using Girsanov's Theorem \cite{girsanov-cameron-and-martin, girsanov1960transforming, sto-calculus-2} by applying the Helmholtz-Hodge decomposition to the stochastic integral of errors and its quadratic variation. 

\subsection{Technical Background} \label{subsec:tech_background}

We work on a complete filtered probability space $(\Omega, \cF, (\cF_t)_{t \in [t_0, T]}, \Prob)$ equipped with a $\Prob$-Brownian motion $(\bB_t)_{t \in [t_0, T]}$. We denote by $(\bar\bB_t)_{t \in [t_0, T]}$ a Brownian motion under the reverse filtration and recall the true reverse SDE, 
\begin{equation}
\label{eq:reverse-sde-true-girsanov-part}
\md\bY_s = [\boldsymbol{f}_s(\bY_s) - \sigma_s^2\,\nabla_{\bx} \log \pstar_s(\bY_s)]\,\md s + \sigma_s\,\md\bar{\bB}_s,
\qquad \bY_T \sim \pstar_T,
\end{equation}
and the learned reverse SDE,
\begin{equation}
\label{eq:reverse-sde-learned-girsanov-part}
\md\hat{\bY}_s
= [\boldsymbol{f}_s(\hat{\bY}_s) - \sigma_s^2\,\nabla_\bx \log \pstar_s(\hat\bY_s) - \sigma_s^2\be_s(\hat\bY_s)]\,\md s + \sigma_s\,\md\bar{\bB}_s,
\qquad \hat{\bY}_T \sim \pstar_T,
\end{equation}
where $\be_s(\bx) := \bs_\theta(\bx, s) - \nabla_\bx \log \pstar_s(\bx)$. 
Under mild (and assumed) regularity assumptions both SDEs are well-posed in law, and we denote $(\bY)_s$ (resp. $(\hat\bY_s)_s$) the weak solution \cite{sto-calculus-2, sto-calculus-le-gall} of \eqref{eq:reverse-sde-true-girsanov-part} (resp. of \eqref{eq:reverse-sde-learned-girsanov-part}). In the original probability space $(\Omega, \cF, (\cF_t)_{t \in [t_0, T]}, \Prob)$, these are two distinct processes (functions of $\Omega \to C([t_0, T], \R^d)$) under a same Brownian motion $\bar{\bB}$. 

\medskip
Based on these two weak solutions, we define $\Prob^\star=\law((\bY_s)_s)$ and $\hat\Prob=\law((\hat\bY_s)_s)$ to be the path laws of \eqref{eq:reverse-sde-true-girsanov-part} and \eqref{eq:reverse-sde-learned-girsanov-part}. The path-space distributions $\Prob^\star$ and $\hat\Prob$ live in the canonical space of paths $C([t_0, T], \R^d)$, and are two different measures on this space. We equip this space with the coordinate process $(\bX_s = \omega(s))_{s\in[t_0,T]}$, which is the only process we define on $C([t_0, T], \R^d)$ and which corresponds to the value of the path at time $s$. On this new path-space, there is one single process $(\bX_s)_{s\in[t_0,T]}$ and two distributions such that, under $\Prob^\star$, $(\bX_s)_{s\in[t_0,T]}$ satisfies, 
\begin{equation}
    \md\bX_s = [\boldsymbol{f}_s(\bX_s) - \sigma_s^2\,\nabla_{\bx} \log \pstar_s(\bX_s)]\,\md s + \sigma_s\,\md\bar{\bB}^{\Prob^\star}_s,
\qquad \bX_T \sim \pstar_T,
\label{eq:X_Pstar}
\end{equation}
and under $\hat\Prob$, $(\bX_s)_{s\in[t_0,T]}$ satisfies, 
\begin{equation}
    \md\bX_s
= [\boldsymbol{f}_s(\bX_s) - \sigma_s^2\,\nabla_\bx \log \pstar_s(\bX_s) - \sigma_s^2\be_s(\bX_s)]\,\md s + \sigma_s\,\md\bar{\bB}^{\hat\Prob}_s,
\qquad \bX_T \sim \pstar_T,
\label{eq:X_Phat}
\end{equation}
where $\bar{\bB}^{\Prob^\star}_s$ (resp. $\bar{\bB}^{\hat\Prob}_s$) is a Brownian motion under $\Prob^\star$ (resp. under $\hat{\Prob}$). 
As discussed below, under suitable assumptions we have $\Prob^\star \ll \hat \Prob$, with the Radon-Nikodym derivative $ \frac{\md\Prob^\star}{\md\hat\Prob}$ given by Girsanov's theorem (see \eqref{eq:Girsanov_ratio}). Furthermore, the processes $\bar{\bB}^{\Prob^\star}_s$ and $\bar{\bB}^{\hat\Prob}_s$ are related as 
\begin{equation}
\label{eq:innovation-shift}
\bar{\bB}^{\Prob^\star}_t = \bar{\bB}^{\hat\Prob}_t - \int_{T}^{t} \sigma_s\,\be_s\,\md s.
\end{equation}

\paragraph{Girsanov density.} Beyond Assumptions~\ref{ass:minimal-assumptions-marginals}--\ref{ass:fisher}, Girsanov-based arguments further require the following assumption:
\begin{enumerate}[label=\textbf{(A\arabic*)}, ref=A\arabic*, leftmargin=*, resume=assumptions]
    \item \label{ass:A4-girsanov}  Define the stochastic integral and its quadratic variation (both random variables),
    \begin{equation}
    \label{eq:MZ_def}
    M_s := \int_{t_0}^s \sigma_u\, \be_u(\bX_u)\cdot\md\bar\bB^{\hat\Prob}_u,
    \qquad
    \langle M\rangle_s := \int_{t_0}^s \sigma_u^2\,\|\be_u(\bX_u)\|_2^2\,\md u,
    \end{equation}
    and the associated Dol\'eans--Dade exponential
    \[
    Z_s := \exp\!\left(M_s - \tfrac{1}{2}\langle M\rangle_s\right).
    \]
    We assume that $(Z_s)_{s\in[t_0, T]}$ is a true $\hat\Prob$-martingale on $[t_0, T]$.
\end{enumerate}

Assumption \ref{ass:A4-girsanov} ensures that the Girsanov transformation yields a valid probability measure (i.e., $\E_{\hat{\Prob}}[Z_T] = 1$). While $Z_s$ is guaranteed to be a local martingale by construction, different sufficient conditions can be found in the literature to guarantee the full martingale status (Novikov's condition \cite{novikov}, Bene\v{s}'s condition \cite{benes}). 

Under Assumptions~\ref{ass:minimal-assumptions-marginals}--\ref{ass:A4-girsanov}, Girsanov's theorem \cite{girsanov-cameron-and-martin, girsanov1960transforming, sto-calculus-2} implies that the Radon--Nikodym derivative of $\Prob^\star$ with respect to $\hat{\Prob}$ is given by the Dol\'eans--Dade exponential \cite{DoleansDade1970} $Z_T$,
\begin{equation}
    \frac{\md\Prob^\star}{\md\hat\Prob} = Z_T = \exp\Big(\int_{t_0}^T \sigma_s \be_s\cdot\md\bar\bB^{\hat\Prob}_s - \tfrac12\int_{t_0}^T \sigma_s^2\|\be_s\|_2^2\,\md s\Big).
    \label{eq:Girsanov_ratio}
\end{equation}

For a similar application of Girsanov's theorem to analyze the sampling quality of diffusion models, we refer the reader to \cite{sampling-is-as-easy-as-learning-the-score}.

\subsection{Standard Girsanov-Based Upper Bound}

To capture errors along the entire true \eqref{eq:reverse-sde-true-girsanov-part} and learned \eqref{eq:reverse-sde-learned-girsanov-part} reverse trajectories, we define 
\begin{align}
\label{eq:pathwise-score-error-energy}
\varepsilon_{\star}^2
:=
\int_{t_0}^T \sigma_s^2  \|
\be_s\|^2_{L^2(\pstar_s)}\,\mathrm ds
, 
\qquad 
\hat \varepsilon^2
:=
\int_{t_0}^T \sigma_s^2 \|
\be_s\|^2_{L^2(\phat_s)}\,\mathrm ds,
\end{align}
where $\pstar_s$ and $\phat_s$ are the time-$s$ marginals of $\Prob^\star$ and $\hat \Prob$, respectively.
We note that Assumptions~\ref{ass:minimal-assumptions-marginals}--\ref{ass:coeff-reg} imply that, 
$$
\varepsilon_\star^2 = \E_{\Prob^\star}\left[\int_{t_0}^T \sigma_s^2 \|\be_s(\bX_s)\|_2^2 \md s \right] < \infty.
$$

We now formally state the standard upper bound,  which Theorem \ref{thm:kl-upper-new} improves upon.

\label{app:standard-girsanov-upper-bound-proof}
\begin{proposition}[Standard Girsanov upper bound]
\label{prop:kl-upper-girsanov}
Assume \ref{ass:minimal-assumptions-marginals}--\ref{ass:A4-girsanov}, then $\Prob^\star \ll \hat{\Prob}$, 
\begin{equation}
\mathrm{KL}(\pstar_{t_0}\|\phat_{t_0})\le \mathrm{KL}(\Prob^\star\,\|\,\hat{\Prob})= \tfrac12\,\varepsilon_{\star}^2.
\label{eq:kl-upper-main}
\end{equation}
\end{proposition}

We provide a short proof for completeness, although the result is standard \cite{sampling-is-as-easy-as-learning-the-score, benton2024nearly, chen2023improved, khelifa2026errorpropagationmodelcollapse}. This proof also clearly indicates the critical points where modifications are needed to obtain our new bound (Theorem \ref{thm:kl-upper-new}).

\begin{proof}
By Assumption \ref{ass:A4-girsanov}, $(Z_s)_{s\in[t_0,T]} = (\exp\!\left(M_s - \tfrac{1}{2}\langle M\rangle_s\right))_{s \in [t_0, T]}$ is a $\hat{\Prob}$-martingale, so in particular $\E_{\hat{\Prob}}[Z_T] = 1$. Therefore, by Girsanov's Theorem, the Radon--Nikodym derivative of $\Prob^\star$ with respect to $\hat{\Prob}$ is given by \eqref{eq:Girsanov_ratio}.

Moreover, recall from Section \ref{subsec:tech_background} that  $\Prob^\star$ and $\hat\Prob$ induce two Brownian motions, $(\bar{\bB}^{\Prob^\star}_s)_s$  and $(\bar{\bB}^{\hat\Prob}_s)_s$, respectively, which are linked by the identity \eqref{eq:innovation-shift}, which we can write as
\begin{equation}
\label{eq:innovation-shift-star}
\md\bar\bB^{\hat\Prob}_s = \md\bar\bB^{\Prob^\star}_s + \sigma_s \be_s\, \md s. 
\end{equation}

Recalling that $\KL(\Prob^\star\,\|\,\hat\Prob) = \E_{\Prob^\star}[\log (\md\Prob^\star/\md\hat\Prob)]$, substituting \eqref{eq:innovation-shift-star} in the log-likelihood ratio (obtained from \eqref{eq:Girsanov_ratio}) yields,
\begin{align*}
\log\frac{\md\Prob^\star}{\md\hat\Prob}
&= \int_{t_0}^T \sigma_s\,\be_s\cdot\md\bar\bB^{\hat\Prob}_s - \tfrac{1}{2}\int_{t_0}^T\sigma_s^2\,\|\be_s\|_2^2\,\md s \\
&= \int_{t_0}^T \sigma_s\,\be_s\cdot\big(\md\bar{\bB}^{\Prob^\star}_s + \sigma_s\,\be_s\,\md s\big) - \tfrac{1}{2}\int_{t_0}^T\sigma_s^2\,\|\be_s\|_2^2\,\md s \\
&= \int_{t_0}^T \sigma_s\,\be_s\cdot\md\bar{\bB}^{\Prob^\star}_s + \tfrac{1}{2}\int_{t_0}^T\sigma_s^2\,\|\be_s\|_2^2\,\md s.
\end{align*}
Taking expectation under $\Prob^\star$,
\begin{equation}
\label{eq:decomposition-likelihood-probstar}
\E_{\Prob^\star}\!\left[\log\frac{\md\Prob^\star}{\md\hat\Prob}\right]
= \E_{\Prob^\star}\!\left[\int_{t_0}^T\sigma_s\,\be_s\cdot\md\bar{\bB}^{\Prob^\star}_s\right] + \tfrac{1}{2}\E_{\Prob^\star}\!\left[\int_{t_0}^T\sigma_s^2\,\|\be_s\|_2^2\,\md s\right].
\end{equation}
The first term is the expectation of a stochastic integral. By Assumptions~\ref{ass:minimal-assumptions-marginals}--\ref{ass:coeff-reg}, the integrand $\sigma_s\be_s$ is square-integrable (since $\varepsilon_\star^2 = \E_{\Prob^\star}\!\left[\int_{t_0}^T\sigma_s^2\,\|\be_s\|_2^2\,\md s\right]< \infty$); as $\bar{\bB}^{\Prob^\star}_s$ is a $\Prob^\star$-Brownian motion, the integral is a true $\Prob^\star$-martingale started at $0$ and its expectation under $\Prob^\star$ vanishes so the first term in the right-hand term of \eqref{eq:decomposition-likelihood-probstar} is zero. 
The second term is exactly $\tfrac{1}{2}\varepsilon^2_\star$, by the definition of $\varepsilon^2_\star$ in \eqref{eq:pathwise-score-error-energy}.

\medskip
\textit{Data Processing Inequality Step.}
Let $\pi_{t_0}: C([t_0, T], \R^d) \to \R^d$ be the projection map $\omega \mapsto \omega(t_0)$. The endpoint marginals are push-forwards: $\pstar_{t_0} = (\pi_{t_0})_\# \Prob^\star$ and $\phat_{t_0} = (\pi_{t_0})_\# \hat\Prob$. By the data-processing inequality (contraction of KL under push-forward),
$$
\KL(\pstar_{t_0}\,\|\,\phat_{t_0}) \le \KL(\Prob^\star\,\|\,\hat\Prob) = \tfrac{1}{2}\varepsilon^2_\star.
$$
\end{proof}

\subsection{A Helmholtz--Hodge View of Girsanov's Theorem}

In this section we explain how the projected bound of
Theorem~\ref{thm:kl-upper-new} can also be recovered from a
Girsanov-type argument. The key point is that Girsanov's theorem is a path-space statement, whereas the Helmholtz--Hodge decomposition identified in Section~\ref{sec:FP-invisible-errors} is intrinsic to marginal dynamics. Consequently, the projected bound is obtained not by applying Girsanov directly to the original learned process, but by first replacing the learned process by a marginally equivalent representative whose drift error is the observable component of the score error.

Recall that the true reverse process has path law $\Prob^\star$ and
satisfies
\begin{equation}
\label{eq:app-true-reverse-girsanov}
\md\bY_s
=
\big[
\boldsymbol f_s(\bY_s)
-
\sigma_s^2\nabla \log \pstar_s(\bY_s)
\big]\md s
+
\sigma_s\,\md\bar{\bB}_s,
\qquad
\bY_T\sim \pstar_T,
\qquad s:T\to t_0 .
\end{equation}
The learned reverse process has path law $\hat \Prob$ and satisfies
\begin{equation}
\label{eq:app-learned-reverse-girsanov}
\md\hat{\bY}_s
=
\big[
\boldsymbol f_s(\hat{\bY}_s)
-
\sigma_s^2\bs_\theta(\hat{\bY}_s,s)
\big]\md s
+
\sigma_s\,\md\bar{\bB}_s,
\qquad
\hat{\bY}_T\sim \pstar_T , \qquad s:T\to t_0 .
\end{equation}
Moreover, plugging $\nabla \log \pstar_s = \bs_\theta(\cdot, s) - \be_s$ in \eqref{eq:app-true-reverse-girsanov} yields, 
\begin{equation}
\label{eq:app-true-reverse-girsanov-v2}
\md\bY_s
=
\big[
\boldsymbol f_s(\bY_s)
-
\sigma_s^2\bs_\theta(\bY_s, s) + \sigma_s^2 \be_s(\bY_s)
\big]\md s
+
\sigma_s\,\md\bar{\bB}_s,
\qquad
\bY_T\sim \pstar_T,
\qquad s:T\to t_0 .
\end{equation}
The drift discrepancy between \eqref{eq:app-learned-reverse-girsanov}
and \eqref{eq:app-true-reverse-girsanov} is given by weighted score estimation errors $\sigma_s^2\be_s$. Recall from the proof of Proposition \ref{prop:kl-upper-girsanov} that

\begin{equation}
\label{eq:app-classical-girsanov-kl}
\KL(\pstar_{t_0}\,\|\,\phat_{t_0}) \le \KL(\Prob^\star\,\|\,\hat\Prob)
=
\frac12
\int_{t_0}^T
\sigma_s^2
\|\be_s\|_{L^2(\pstar_s)}^2
\,\md s .
\end{equation}
The looseness of this bound comes from the fact that it considers all $\|\be_s\|_{L^2(\pstar_s)}^2$ errors, even if not every component of $\be_s$ actually affects marginals.

\paragraph{Bridge between Girsanov's approach and ours.} To bridge this gap, we decompose the error into its observable gradient component and its invisible solenoidal component, as explained in Section \ref{sec:FP-invisible-errors}. Formally, at each time $s$,
$$
\be_s
=
\Pi_{\cGs}\be_s
+
\Pi_{\cGs^\perp}\be_s,
$$
where we recall that $\cGs := \mathrm{cl}\{\nabla \varphi:\varphi\in C_c^\infty(\R^d)\}$, with the closure taken with respect to $L^2(\pstar_s;\R^d)$. We now construct a marginally equivalent representative in which these invisible directions are removed. As discussed in Section \ref{sec:FP-invisible-errors}, this decomposition implies  $\nabla\cdot(\pstar_s\be_s) = \nabla\cdot\!\big(\pstar_s\Pi_{\cGs}\be_s\big)$, which allows us to rewrite the true marginal dynamics associated with the reverse SDE \eqref{eq:app-true-reverse-girsanov-v2} as in \eqref{eq:FP-HH-learned-reverse}, which we restate for convenience: 

\begin{equation}
\label{eq:app-observable-FP}
\partial_s \pstar_s
=
-\nabla\cdot
\Big(
[
\boldsymbol f_s
-
\sigma_s^2 \bs_\theta(\cdot, s)
]\pstar_s
\Big)
-
\frac{\sigma_s^2}{2}\Delta \pstar_s
-
\sigma_s^2\nabla\cdot
\big(
\pstar_s\Pi_{\cGs}\be_s
\big).
\end{equation}

Now introduce a process
$\tilde{\bY} = (\tilde{\bY}_s)_{s\in[t_0, T]}$ that is the solution to  the following SDE
\begin{equation}
\label{eq:app-observable-representative}
\md\tilde{\bY}_s
=
\big[
\boldsymbol f_s(\tilde{\bY}_s)
-
\sigma_s^2\bs_\theta(\tilde{\bY}_s, s)
+
\sigma_s^2\Pi_{\cGs}\be_s(\tilde{\bY}_s)
\big]\md s
+
\sigma_s\,\md\bar{\bB}_s,
\qquad
\tilde{\bY}_T\sim \pstar_T.
\end{equation}
Then the marginals $(\tilde{p}_s)_{s \in [t_0, T]}$, where $\tilde p_s:=\law(\tilde{\bY}_s)$, solve the Fokker-Planck equation \eqref{eq:app-observable-FP}. 
We denote by $\tilde\Prob$ the path-space law of \eqref{eq:app-observable-representative}.  

On the other hand, since the true marginals $\pstar_s$ solve the Fokker-Planck associated with the reverse SDE \eqref{eq:app-true-reverse-girsanov-v2}, we have
\begin{equation}
\label{eq:fp-true-learned-reverse-v2}
\partial_s \pstar_s
=
-\nabla\cdot
\Big(
[
\boldsymbol f_s
-
\sigma_s^2\bs_\theta(\cdot, s)
]\pstar_s
\Big)
-
\frac{\sigma_s^2}{2}\Delta \pstar_s
-
\sigma_s^2\nabla\cdot
(\pstar_s\be_s).
\end{equation}
But, since $\nabla\cdot(\pstar_s\be_s) = \nabla\cdot\!\big(\pstar_s\Pi_{\cGs}\be_s\big)$ and terminal conditions match, the marginal dynamics \eqref{eq:fp-true-learned-reverse-v2} and \eqref{eq:app-observable-FP} exactly match, i.e.,  $\tilde{p}_s = \pstar_s$ for all $s \in [t_0, T]$.

In particular, endpoint marginals match $\tilde{p}_{t_0}=\pstar_{t_0}$. 

We may now apply Girsanov's theorem between $\hat{\Prob}$ and the observable representative $\tilde{\Prob}$ exactly as is done in Proposition \ref{prop:kl-upper-girsanov}, replacing $\Prob^\star$ by $\tilde{\Prob}$. The SDEs \eqref{eq:app-learned-reverse-girsanov}  and \eqref{eq:app-observable-representative} have the same diffusion coefficient and the same terminal law $\pstar_T$; their drift discrepancy is $\sigma_s^2\Pi_{\cGs}\be_s$. Assuming the analogue of~\ref{ass:A4-girsanov} for $\Pi_{\cGs}\be_s$, the corresponding Dol\'eans--Dade exponential (obtained from \eqref{eq:MZ_def} by replacing $\be_s$ with $\Pi_{\cGs} \be_s $ ) is a true martingale, Girsanov's formula gives
\begin{equation}
\label{eq:app-projected-path-kl}
\KL(\tilde\Prob\,\|\,\hat{\Prob})
=
\frac12
\int_{t_0}^T
\sigma_s^2
\|\Pi_{\cGs}\be_s\|_{L^2(\tilde p_s)}^2
\,\md s .
\end{equation}
Using \eqref{eq:app-projected-path-kl}, $\tilde{p}_{t_0}=\pstar_{t_0}$ and data-processing, this becomes,
\begin{equation}
\label{eq:app-projected-path-kl-phat}
\KL(\pstar_{t_0}\;\|\;\phat_{t_0}) 
= 
\KL(\tilde{p}_{t_0}\;\|\;\phat_{t_0}) 
\le 
\KL(\tilde{\Prob}\;\|\;\hat{\Prob}) 
=
\frac12
\int_{t_0}^T
\sigma_s^2
\|\Pi_{\cGs}\be_s\|_{L^2(\pstar_s)}^2
\,\md s,
\end{equation}
which is exactly the bound of Theorem~\ref{thm:kl-upper-new}.

\section{Optimal Transport Interpretation}
\label{app:OT-interpretation}

The variational characterization to compute the gradient component of errors, given in Proposition \ref{prop:grad-error-minimizes-L2} by, 
\begin{equation}
    \label{eq:grad-error-minimizes-L2-v2}
    \|\Pi_{\cGs}\be_s\|_{L^2(\pstar_s)}^2
    =
    \inf_{\bv_s:\, \nabla\cdot(\pstar_s\bv_s)=\nabla\cdot(\pstar_s\be_s)}
    \int_{\mathbb R^d} \|\bv_s(\bx)\|_2^2\,\pstar_s(\bx)\,\md\bx.
\end{equation}
has an interpretation in the dynamic formulation of optimal transport. Recall $\Probspace_2(\R^d)$ denotes the space of probability densities on $\R^d$ with finite second moment. Let $\mu\in\mathcal P_2(\R^d)$ be a smooth positive density, and consider a differentiable curve of probability
densities $(\mu_\tau)_{\tau\in(-\epsilon,\epsilon)}$ such that
$\mu_0=\mu$. Its tangent vector at $\tau=0$ is the signed density
$$
\zeta
:=
\left.\partial_\tau \mu_\tau\right|_{\tau=0},
\qquad
\int_{\R^d}\zeta(\bx)\,\md\bx=0.
$$
In the Benamou--Brenier geometry \cite{benamou-brenier}, this tangent perturbation is represented by a velocity field $\bv\in L^2(\mu;\R^d)$ through the continuity equation \cite{ambrosioGradientFlowsMetric2008, santambrogio-ot,  chewi2025statistical}
$$
\zeta
=
-\nabla\cdot(\mu\bv).
$$
This representation is not unique: if $\bw$ satisfies
$\nabla\cdot(\mu\bw)=0$, then $\bv+\bw$ represents the same tangent vector $\zeta$. Wasserstein geometry removes this ambiguity by selecting the minimum-kinetic-energy representative,
$$
\inf_{\bv:\,-\nabla\cdot(\mu\bv)=\zeta}
\int_{\R^d}\|\bv(\bx)\|_2^2\,\mu(\bx)\,\md\bx.
$$
Equivalently, this minimum-energy representative belongs to the closure of gradient fields \cite{ambrosioGradientFlowsMetric2008, villani_OT, chewi2025statistical},
$$
T_\mu\mathcal P_2
=
\mathrm{cl}(\{\nabla\varphi:\varphi\in C_c^\infty(\R^d)\}),
$$ 
where the closure $\mathrm{cl}$ is taken in $L^2(\mu;\R^d)$. The Otto metric \cite{otto} equips this tangent space with the norm obtained by pulling back the $L^2(\mu)$ geometry through the minimal kinetic energy gradient representative,
$$
\|\zeta\|_{T_\mu\cP_2}^2
= \inf_{\bv:\,-\nabla\cdot(\mu\bv)=\zeta}\int_{\R^d}\|\bv\|_2^2\,\mu\,\md\bx.
$$
This is the $H^{-1}(\mu)$-norm \cite{Brezis2010-sobolev-pde} of the signed density $\zeta$.  To recover the setting of Section \ref{sec:from-observable-to-computable}, fix a diffusion time $s \in [t_0,T]$, take $\mu=\pstar_s$, and let
$\zeta_s := -\nabla\cdot(\pstar_s\be_s)$ be the tangent vector generated by the score error (in the Fokker--Planck equation \eqref{eq:fp-true-learned-reverse-v2} the error contributes $-\sigma_s^2\,\nabla\cdot(\pstar_s\be_s)$ to $\partial_s\pstar_s$). Then, by the
definition above together with Proposition~\ref{prop:grad-error-minimizes-L2},
$$
\|\zeta_s\|_{T_{\pstar_s}\cP_2}^2
=\inf_{\bv:\,\nabla\cdot(\pstar_s\bv)=\nabla\cdot(\pstar_s\be_s)}
\int_{\R^d}\|\bv\|_2^2\,\pstar_s\,\md\bx
=\|\Pi_{\cGs}\be_s\|_{L^2(\pstar_s)}^2 .
$$
Equivalently, this tangent norm is the $H^{-1}(\pstar_s)$ norm
$\|\nabla\cdot(\pstar_s\be_s)\|_{H^{-1}(\pstar_s)}$ of \eqref{def:Hminus-1}: the Otto metric tensor and the $H^{-1}(\pstar_s)$ norm coincide. Thus, the target-weighted observable error is the squared Wasserstein tangent norm of the infinitesimal curve of marginals generated by the score error by the velocity field given by the score error at $\pstar_s$. The solenoidal component of $\be_s$ is invisible because it changes the velocity representation without changing the tangent vector $\zeta_s$, i.e. without changing the infinitesimal motion of probability mass.

\section{Experimental Setting}
\label{app:experiments}

\subsection{Experimental Setup Fashion-MNIST}
\label{app:two_phase_setup_fmnist}

We detail the experimental setup used to produce the Fashion-MNIST error-decomposition curves, KL upper-bound estimates, and Spearman correlations of Section~\ref{sec:FP-invisible-errors}. All image experiments (on CIFAR-10 and Fashion-MNIST) were run on a single NVIDIA RTX 6000 Blackwell GPU with 96GB memory; a full CIFAR-10 score network training run required approximately 10 hours of wall-clock time. 

\paragraph{Training setup.}
We train standard score-based diffusion models on Fashion-MNIST \cite{fashionMNIST_paper} using a variance-preserving SDE with linear schedule $\beta(t) = \beta_{\min} + t(\beta_{\max} - \beta_{\min})$, $\beta_{\min} = 0.1$, $\beta_{\max} = 20$, time horizon $T = 4$ and truncation $t_0 = 5\times 10^{-2}$. The score network $\bs_\theta$ is a three-resolution-level U-Net with self-attention~\cite{unet-seminal,vaswani2017attention} parametrized by a base channel multiplier $c$ and a per-resolution depth $n_{\mathrm{res}}$. We sweep three capacity tiers, \emph{tiny} ($c{=}32$, $n_{\mathrm{res}}{=}1$, $1.76$M parameters), \emph{small} ($c{=}48$, $n_{\mathrm{res}}{=}2$, $3.92$M) and \emph{full} ($c{=}64$, $n_{\mathrm{res}}{=}2$, $6.27$M), and train each for $400$ epochs with AdamW~\cite{loshchilov2018decoupled} (learning rate $5\times 10^{-4}$ with linear warmup over $2{,}000$ steps and cosine decay, weight decay $10^{-4}$, batch size $512$, gradient clipping at $1.0$). We use a log-uniform time-sampling distribution over $[t_0, T]$, an exponential moving average of the score-network weights with decay $0.9999$ used for all evaluations, and five independent seeds per capacity. Permanent checkpoints are saved every $20$ epochs, yielding $20$ checkpoints per run.

\paragraph{Estimating the error decomposition.}
At each saved checkpoint we evaluate the gradient and solenoidal components of the score error on the EMA score network via the computational procedure of Section~\ref{subsec:computational-aspect}. The full $L^2$ error is computed directly from the unbiased residual $\br_\theta(\bX_s,s) = \shat_\theta(\bX_s,s) - \bs_{\mathrm{target}}(\bX_s,s)$ defined in equation~\eqref{eq:dsm_target_conditional}, by Monte Carlo averaging over a fixed pool of $6{,}400$ noised samples drawn from the Fashion-MNIST training set with $t$ sampled from the same log-uniform distribution used during training (the same $(\bX_0, t, \bepsilon, \bX_t)$ tuples are reused across all checkpoints of a run to remove sampling noise from the curves). The gradient component is estimated via the dual variational identity of equation~\eqref{eq:Hminus1_force_dual}: at each evaluation, we instantiate a freshly initialized critic potential $\varphi_\psi$ and train it for $1{,}500$ steps to maximize the empirical objective $\widehat{\cJ}(\psi)$ of equation~\eqref{eq:J_theta_psi_emp}, using Adam~\cite{Kingma2014AdamAM} (learning rate $10^{-3}$ with cosine decay to $10^{-5}$, gradient clipping at $1.0$, batch size $256$). The critic architecture mirrors the score network's encoder–decoder structure but with reduced width (base channel multiplier $24$, scalar output, single attention block at the bottleneck), and its gradient $\nabla_{\bx} \varphi_\psi$ is taken via automatic differentiation. The solenoidal component is obtained as $\|\Pi_{\cGs^\perp}\be_s\|^2 = \max(\|\be_s\|^2 - \|\Pi_{\cGs}\be_s\|^2, 0)$, with the floor at zero handling the rare cases where finite-sample noise makes the dual lower bound exceed the empirical full norm. For the final-checkpoint tight bound used in Theorem~\ref{thm:kl-upper-new}, we re-estimate the gradient component with a $2\times$ wider critic (base channel multiplier $48$) trained for $5{,}000$ steps with three independent restarts, reporting the across-restart mean and standard deviation.

\paragraph{KL upper-bound estimates.}
The right panel of Figure~\ref{fig:kl-bounds-fmnist} reports both Girsanov-style and Helmholtz–Hodge bounds on $\KL(\pstar_{t_0}\|\phat_{t_0})$. The two bounds, equations~\eqref{eq:kl-upper-old} and~\eqref{eq:kl-upper-new}, are respectively the time integrals $\tfrac{1}{2}\int_{t_0}^T \sigma_s^2 \,\|\be_s\|^2_{L^2(\pstar_s)}\,\md s$ and $\tfrac{1}{2}\int_{t_0}^T \sigma_s^2 \,\|\Pi_{\cGs}\be_s\|^2_{L^2(\pstar_s)}\,\md s$. As in the CIFAR-10 setup (described in the following section), we use a common multiplicative constant $K = (T - t_0)\bar{\sigma}^2 / 2$ where $\bar{\sigma}^2$ is the schedule-averaged $\sigma_s^2$ over $[t_0, T]$. The constant $K$ is identical for both curves and does not depend on the score network, epoch, or seed, so the ratio of the two curves --- and therefore the inobservability gap --- is faithful to theoretical predictions independently of $K$.

\paragraph{Sample-quality metric.}
Sample quality is measured by feature-FID computed in the $128$-dimensional penultimate-layer feature space of a small CNN classifier \cite{lecun1998gradient} ($> 92\%$ test accuracy on Fashion-MNIST), evaluated on $5{,}000$ samples generated via a $200$-step predictor–corrector reverse SDE with the EMA score network. Because the diagnostic and FID are computed from independent quantities (training residual vs.\ generated samples), no leakage is possible between the two columns of Figure \ref{fig:spearman_fmnist}.

\paragraph{Spearman rank correlation.}
Per-seed Spearman correlations are computed across the $20$ checkpoints of a run between feature-FID and each of $\|\be\|^2$, $\|\Pi_\cG\be\|^2$. Pooled correlations are computed across the resulting $60$ (seed, epoch) pairs at each capacity. We report Spearman rather than Pearson because both error norms and FID range over multiple orders of magnitude during training, so a rank-based statistic is more meaningful than a linear one. Across all nine (capacity, seed) combinations the gradient-component correlation lies in $[0.93, 0.99]$, whereas the full-error correlation ranges as low as $0.65$, indicating that $\|\Pi_\cG\be\|^2$ is not only a better but also a substantially more reliable proxy for sample quality than the full $L^2$ score error.

\paragraph{Critic-suboptimality control.}
The dual estimator of $\|\Pi_{\cGs}\be_s\|^2$ is, at finite critic capacity, a lower bound. To verify that the apparent solenoidal plateau is not an artifact of critic underfitting, we ran an ablation on the \emph{small} capacity tier in which the eval-time critic was trained for $4{,}500$ steps ($3\times$ longer) and with a $2\times$ wider architecture (base channel multiplier $48$). The resulting gradient-component curve was within the inter-quartile band of the reported one, and the solenoidal plateau was unchanged at the displayed log-scale resolution. We therefore interpret the plateau as a genuine property of the score network rather than as a critic-capacity artifact.

\subsection{Additional Results Fashion-MNIST}
\label{sec:additional-results-fmnist}

\begin{figure}[ht]
\centering
\includegraphics[width=0.5\linewidth]{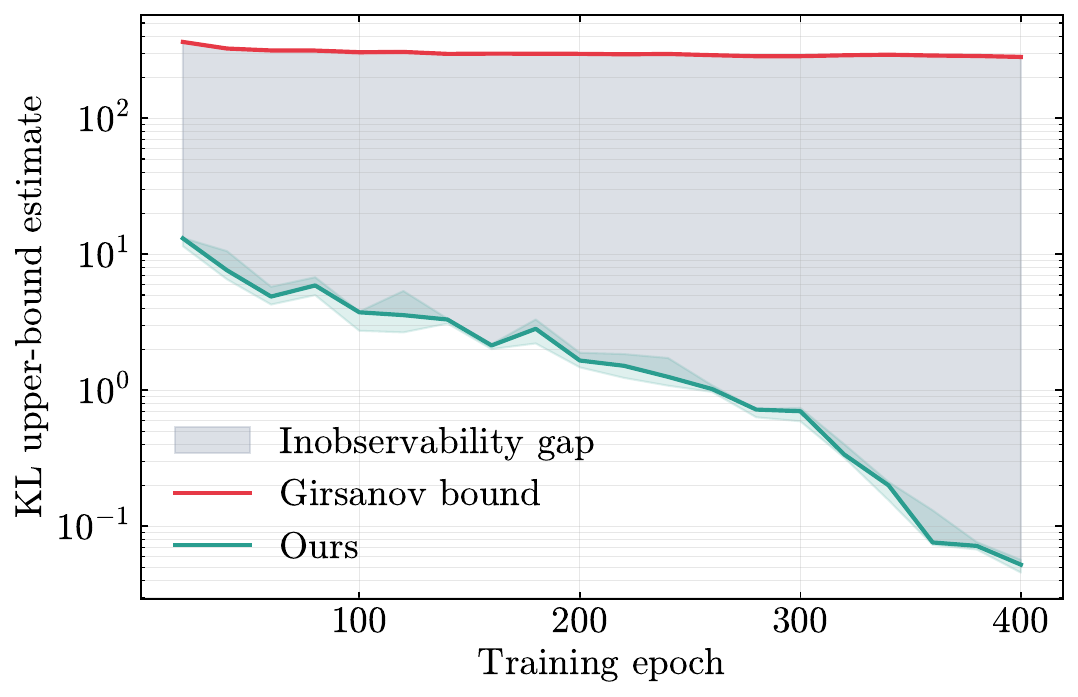}\hfill
\includegraphics[width=0.5\linewidth]{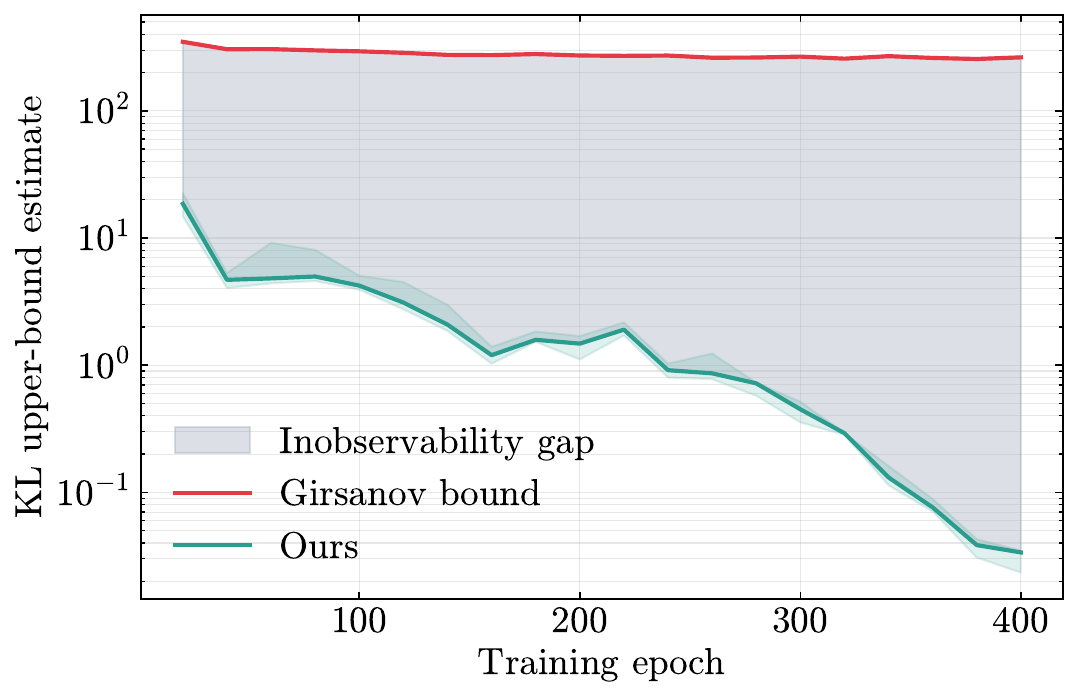}\hfill
\includegraphics[width=0.5\linewidth]{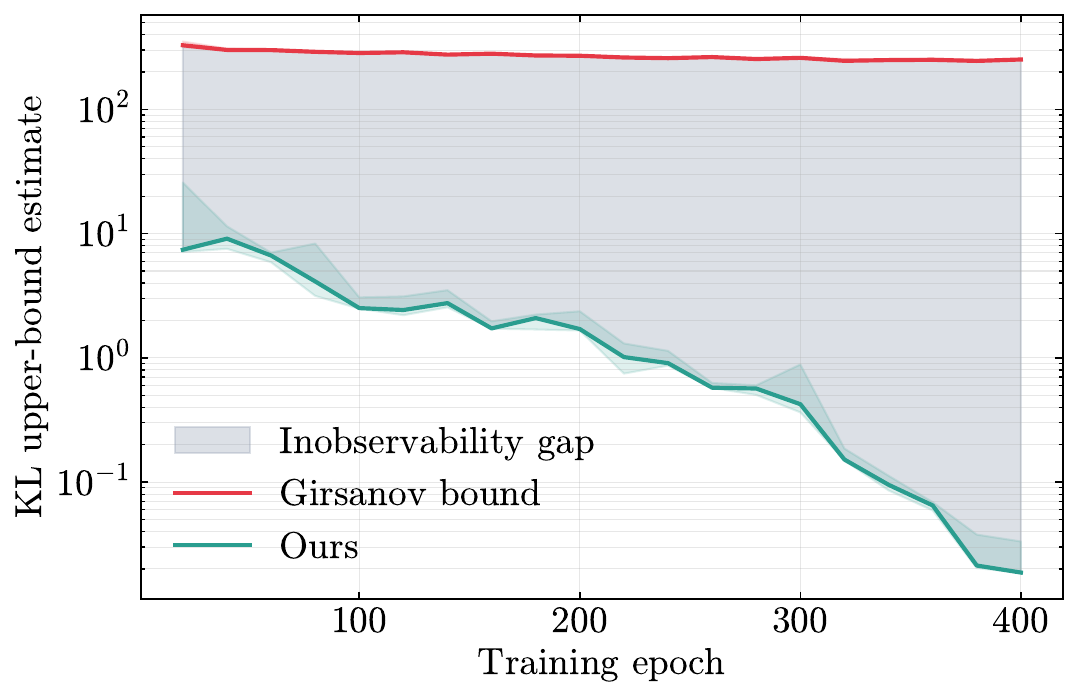}
\caption{\textbf{The Girsanov bound saturates while ours keeps decreasing, across all model capacities.} KL upper-bound estimates on Fashion-MNIST for \emph{tiny}-sized (top left) \emph{small}-sized (top right) and \emph{normal}-sized (bottom) score networks. The Girsanov bound~\eqref{eq:kl-upper-old} stays effectively constant at the solenoidal floor, whereas our bound from Theorem~\ref{thm:kl-upper-new} tracks the gradient component and decreases by roughly three orders of magnitude. The shaded \emph{inobservability gap} is the slack the standard analysis incurs by penalizing error components that, by Theorem~\ref{thm:observable-score-error-principle}, do not affect generated marginals. Median (solid) and inter-quartile range (band) over three seeds. Setup, sampling protocol, and the choice of bound prefactor in Appendix~\ref{app:two_phase_setup_fmnist}.}
\label{fig:kl-bounds-fmnist}
\end{figure}

The same qualitative behavior holds on Fashion-MNIST across three capacity tiers (\emph{tiny}, \emph{small}, \emph{full}): the gradient component of the score error decreases by more than three orders of magnitude during training while the solenoidal component remains almost flat and dominates the total error throughout (Figure~\ref{fig:error-decomp-fmnist}, Appendix~\ref{app:two_phase_setup_fmnist}); the standard Girsanov bound therefore stays effectively constant, and is a loose proxy of learning quality while our bound from Theorem~\ref{thm:kl-upper-new} tracks sample quality and continues to decrease (Figure~\ref{fig:kl-bounds-fmnist}). This matches the Spearman correlation results between the FID and the two error components in Figure \ref{fig:spearman_fmnist}. We deliberately let the model size as well as the dataset vary (see results on CIFAR-10 in Figure \ref{fig:two-phase-and-bounds-cifar10}) to illustrate that the asymmetry between observable and invisible components is not an artifact of architecture, dataset, or model capacity, but a structural property of standard score matching predicted by Theorem~\ref{thm:observable-score-error-principle}.

\begin{figure}[ht]
\centering
\includegraphics[width=0.5\linewidth]{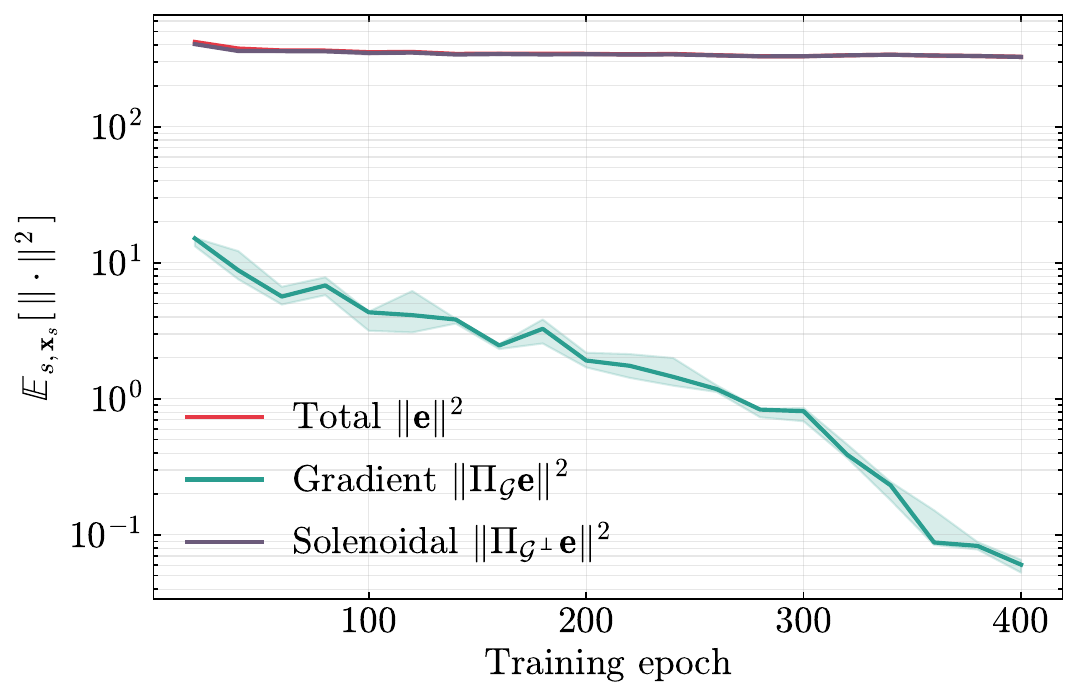}\hfill
\includegraphics[width=0.5\linewidth]{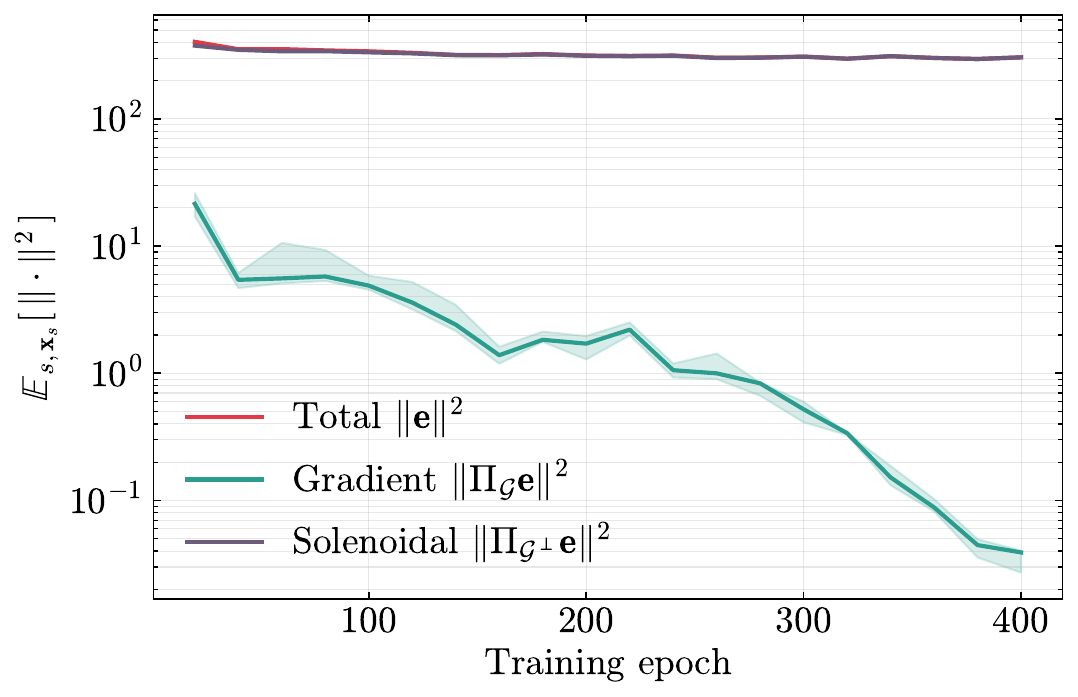}\hfill
\includegraphics[width=0.5\linewidth]{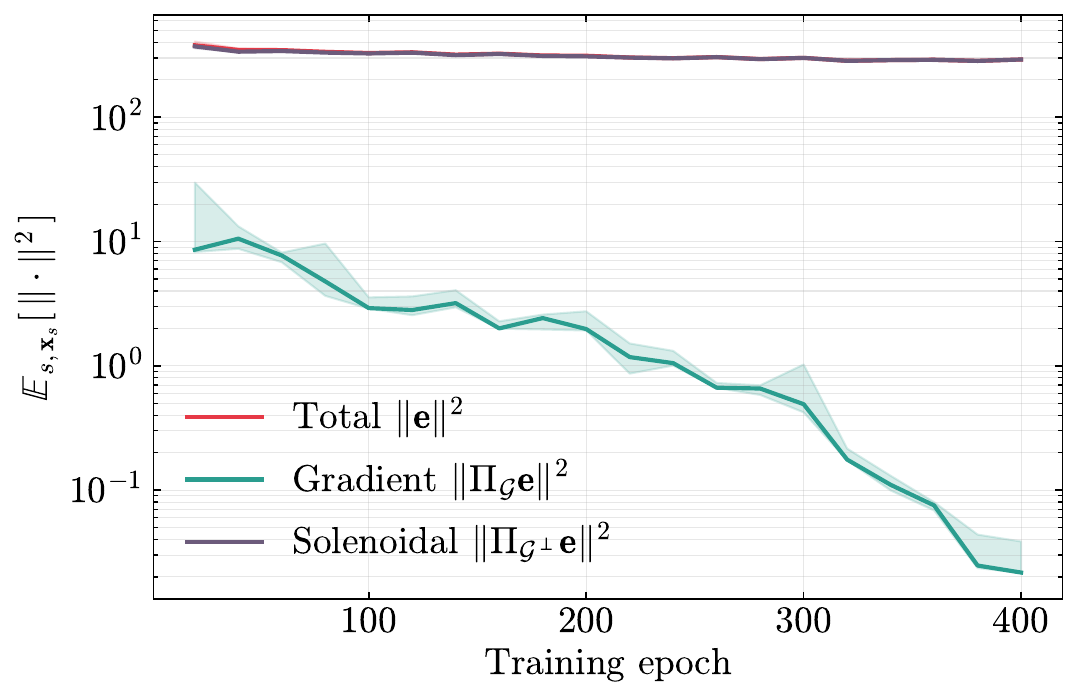}
\caption{\textbf{DSM does not reduce the invisible component of the score error, across all model capacities.} Decomposition of the score estimation error $\be_s$ into its observable gradient component $\Pi_{\cGs}\be_s$ and its invisible solenoidal component $\Pi_{\cGs^\perp}\be_s$ on Fashion-MNIST, for \emph{tiny}-sized (top left) \emph{small}-sized (top right) and \emph{normal}-sized (bottom) score networks. Across all capacities, the gradient component falls by more than three orders of magnitude during training, while the solenoidal component remains essentially flat and accounts for nearly all of the residual error throughout. Median (solid) and interquartile range (band) over five seeds. Setup in Appendix~\ref{app:two_phase_setup_fmnist}.}
\label{fig:error-decomp-fmnist}
\end{figure}

The same separation between observable and invisible error components, and the same gap between the standard and observable KL bounds, hold on Fashion-MNIST across three capacity tiers (Appendix~\ref{app:two_phase_setup_fmnist}, Figures~\ref{fig:error-decomp-fmnist}--\ref{fig:kl-bounds-fmnist}).

\subsection{Experimental Setup CIFAR-10}
\label{app:two_phase_setup}

\begin{algorithm}[t]
\caption{Observable-error critic estimation}
\label{alg:observable-error-critic}
\begin{algorithmic}[1]
\REQUIRE Frozen score network $\shat_\theta$; critic network $\varphi_\psi$; training dataset $\cD \sim \pstar$; time sampler $\nu$ over $[t_0,T]$; noising schedules $\alpha(s),\sigma(s)$; batch size $B$; critic steps $N_\mathrm{crit}$; learning rate $\eta$.
\ENSURE Estimate $\widehat{\cL}_{\mathrm{obs}}$ of
$\E_{s\sim\nu}\|\Pi_{\cGs}\be_s\|^2_{L^2(\pstar_s)}$.

\STATE Initialize critic parameters $\psi$.
\FOR{$k=1,\dots,N_\mathrm{crit}$}
    \STATE $\{\bx_0^{(i)}\}_{i=1}^B \sim \cD$; $\{s^{(i)}\}_{i=1}^B \sim \nu$; $\{\bepsilon^{(i)}\}_{i=1}^B \sim \cN(0,\bI_d)$
    \STATE $\bx_s^{(i)}
    \leftarrow
    \alpha(s^{(i)})\bx_0^{(i)}
    +
    \sigma(s^{(i)})\bepsilon^{(i)}$
    \STATE $\bs_{\mathrm{target}}^{(i)}
    \leftarrow
    -\dfrac{\bx_s^{(i)}-\alpha(s^{(i)})\bx_0^{(i)}}
    {\sigma(s^{(i)})^2}$
    \STATE $\br_\theta^{(i)}
    \leftarrow
    \bs_\theta(\bx_s^{(i)},s^{(i)})
    -
    \bs_{\mathrm{target}}^{(i)}$
    \STATE $\bg_\psi^{(i)}
    \leftarrow
    \nabla_{\bx}\varphi_\psi(\bx_s^{(i)},s^{(i)})$
    \STATE $\widehat{\cJ}_B(\psi)
    \leftarrow
    \dfrac{1}{B}\sum_{i=1}^B
    \left[
    2\,\bg_\psi^{(i)}\cdot \br_\theta^{(i)}
    -
    \|\bg_\psi^{(i)}\|_2^2
    \right]$
    \STATE $\psi \leftarrow \psi + \eta\,\nabla_\psi \widehat{\cJ}_B(\psi)$
\ENDFOR

\STATE Draw a fresh validation batch and compute $\widehat{\cJ}_{\mathrm{val}}(\psi)$
by repeating lines 3--8 (no update on $\psi$). 
\STATE $\widehat{\cL}_{\mathrm{obs}} \leftarrow \widehat{\cJ}_{\mathrm{val}}(\psi)$
\RETURN $\widehat{\cL}_{\mathrm{obs}}$
\end{algorithmic}
\end{algorithm}

We detail the experimental setting used to produce the error decomposition and KL upper-bound curves of Figure~\ref{fig:two-phase-and-bounds-cifar10}.

\paragraph{Training setup.}
We train a standard score-based diffusion model on CIFAR-10 with the variance-preserving SDE of \cite{song2021scorebased}. With our notations in Section \ref{sec:background}, we choose $\boldsymbol{f}_t(\bx) = -\tfrac{1}{2}\beta(t)\bx$ and $\sigma_t = \sqrt{\beta(t)}$, with a linear schedule $\beta(t) = \beta_{\min} + t(\beta_{\max} - \beta_{\min})$, $\beta_{\min} = 0.1$, $\beta_{\max} = 20$, and time horizon $T=1$ with truncation at $t_0 = 0.01$. We parametrize the score network $\bs_\theta$ via a three-resolution level U-Net \cite{unet-seminal} with self-attention \cite{vaswani2017attention} (approximately $12.6$M trainable parameters). Models are trained for $350$ epochs with AdamW \cite{loshchilov2018decoupled} (learning rate $2\times 10^{-4}$ with linear warm-up over $3000$ steps and cosine decay to $10^{-5}$, weight decay $0.01$, batch size $128$, gradient clipping at $1.0$). We use Min-SNR-$\gamma$ loss weighting \cite{minSNR} with $\gamma = 5$, and maintain an exponential moving average of the score-network weights with decay $0.9999$ that is used for all evaluations. Curves are reported as median and inter-quartile range over four independent seeds.

\paragraph{Estimating the error decomposition.}
Every $10$ epochs, we evaluate the gradient and solenoidal components of the score error on the EMA score network via the computational procedure described in Section \ref{subsec:computational-aspect}. The full $L^2$ error is computed directly from the unbiased residual $\br_\theta(\bX_s,s) = \shat_\theta(\bX_s,s) - \bs_{\mathrm{target}}(\bX_s,s)$ defined in equation~\eqref{eq:dsm_target_conditional}, by Monte Carlo averaging over a fixed pool of $\sim 6{,}400$ noised samples drawn from the CIFAR-10 training set with $t \sim \mathcal{U}[t_0, T]$ (the same $(\bX_0, t, \bepsilon, \bX_t)$ tuples are reused across all evaluations to remove sampling noise from the curves). The gradient component is estimated via the dual variational identity of equation~\eqref{eq:Hminus1_force_dual}: at each evaluation, we instantiate a freshly initialized critic potential $\varphi_\psi$ and train it for $4{,}000$ steps to maximize the empirical objective $\widehat{\cJ}(\psi)$ in equation~\eqref{eq:J_theta_psi_emp}, using Adam \cite{Kingma2014AdamAM} (learning rate $10^{-3}$ with cosine decay to $10^{-5}$, gradient clipping at $1.0$). The critic architecture mirrors the score network's encoder-decoder structure but with reduced width (base channel multiplier $96$, scalar output), and its gradient $\nabla_\bx \varphi_\psi$ is taken via automatic differentiation. To stabilize the estimator we average over three independent critic restarts per evaluation; the standard deviation of the gradient estimate across restarts is reported alongside the mean. The solenoidal component is then obtained as $\|\Pi_{\cGs^\perp}\be_s\|^2 = \max(\|\be_s\|^2 - \|\Pi_{\cGs}\be_s\|^2, 0)$, with the floor at zero handling the rare cases where finite-sample noise makes the dual lower bound exceed the empirical full norm.

\paragraph{KL upper-bound estimates.}
The right panel of Figure~\ref{fig:two-phase-and-bounds-cifar10} reports both Girsanov-style and Helmholtz--Hodge bounds on $\KL(\pstar_{t_0}\|\phat_{t_0})$. The two bounds, equations~\eqref{eq:kl-upper-old} and~\eqref{eq:kl-upper-new}, are respectively the time integrals $\tfrac{1}{2}\int_{t_0}^T \sigma_s^2 \,\|\be_s\|^2_{L^2(\pstar_s)}\,\md s$ and $\tfrac{1}{2}\int_{t_0}^T \sigma_s^2 \,\|\Pi_{\cGs}\be_s\|^2_{L^2(\pstar_s)}\,\md s$. For simplicity, we chose the same constant for these two bounds, corresponding to $K = (T - t_0) \bar{\sigma}^2/2$ where $\bar{\sigma}^2$ is the schedule-averaged $\sigma_s^2$ over $[t_0, T]$. The constant $K$ is identical for both curves and does not depend on the score network, the epoch, or the seed, making the ratio of the two curves at any epoch (and therefore the inobservability gap) faithful to theoretical predictions. 

\paragraph{Critic-suboptimality control.}
The dual estimator of $\|\Pi_{\cGs}\be_s\|^2$ is, at finite critic capacity, a lower bound. To verify that the apparent solenoidal plateau is not an artifact of critic underfitting, we ran an ablation in which the eval-time critic was trained for $12{,}000$ steps (3$\times$ longer) and with a wider architecture (base channel multiplier $192$ instead of $96$). The resulting gradient-component curve was within the interquartile band of the reported one, and the solenoidal plateau was unchanged at the displayed log-scale resolution. We therefore interpret the plateau as a genuine property of the score network rather than as a critic-capacity artifact.

\subsection{Low-Dimensional Experiment with Exact Decompositions}
\label{app:lowdim}
The target is a Gaussian mixture with $8$ modes on a circle of radius $2$ with uneven weights so that forward marginals and scores are available in closed form. The forward SDE is set to a variance-preserving schedule with
$\beta_{\min}=0.1$, $\beta_{\max}=20$, $T=1$, and $t_0$ chosen so that the marginal noise level at $t_0$ is $0.05$ ($t_0\approx0.0116$). The perturbed mixtures $\tilde p$ are obtained by adding $\cN(0,0.6^2)$ noise to the log-weights and $\cN(0,0.15^2)$ noise to the means (four draws). We add both purely solenoidal and purely gradient error fields. Both the standard Girsanov bound and ours are evaluated by Gauss--Legendre quadrature in $\log s$. The learned endpoint $\phat_{t_0}$ is obtained by solving the reverse Fokker--Planck
equation~\eqref{eq:FP-learned-reverse} from $\phat_T=\pstar_T$ on a $224^2$ grid over $[-4.5,4.5]^2$, with $200$ steps. KL is computed against a reference solve with the exact score on the same grid and time steps, so that discretisation errors largely cancel. The reference solve differs from the analytic $\pstar_{t_0}$ by $10^{-5}$. 

\end{document}